\documentclass{article}

\usepackage[accepted]{icml2018}

\usepackage{chngcntr}
\usepackage{wrapfig}
\usepackage{graphicx, enumitem}
\usepackage{amssymb, amsmath}
\usepackage{epstopdf}
\usepackage{float}
\usepackage{color}
\usepackage{centernot,verbatim}
\usepackage{tikz}
\usepackage{subcaption}
\usepackage{amsthm}
\usepackage{microtype}
\usepackage{graphicx}
\usepackage{booktabs} 
\usepackage{tikz}
\usepackage{algorithm}
\usepackage{algorithmic}
\usepackage{setspace}
\usepackage{xr}
\usepackage{natbib}
\usepackage[titletoc]{appendix}

\newtheorem{thm}{Theorem}[section]
\newtheorem{lemma}[thm]{Lemma}

\newtheorem{prop}[thm]{Proposition}
\newtheorem{cor}[thm]{Corollary}

\newtheorem{defn}[thm]{Definition}

\newtheorem{Example}{Example}

\usetikzlibrary{arrows.meta}

\DeclareMathOperator{\an}{an}
\DeclareMathOperator{\de}{de}
\DeclareMathOperator{\ch}{ch}
\DeclareMathOperator{\pa}{pa}

    \newcommand{\up}[1]{^{(#1)}}
\newcommand{\cD}{\mathcal{D}}

\newcommand{\cG}{\mathcal{G}}

\newcommand{\cM}{\mathcal{M}}
\newcommand{\cN}{\mathcal{N}}

\newcommand{\cP}{\mathcal{P}}

\newcommand{\ci}{\mathrel{\perp\mspace{-10mu}\perp}}
\newcommand{\nci}{\centernot{\ci}}
\newcommand{\incomp}{\not\lessgtr}

\let\and\relax
\newcommand{\and}{~\textrm{and}~}

\newcommand{\D}{\mathcal{D}}
\newcommand{\M}{\mathcal{M}}

\def\leftarrowcirc{\hbox{$\leftarrow$}\kern-1.5pt\hbox{$\circ$}}
\def\rightarrowcirc{\hbox{$\circ$}\kern-1.5pt\hbox{$\rightarrow$}}
\def\circtailcirc{\hbox{$\circ$}\kern-1.5pt\hbox{$-$}\kern-1.5pt\hbox{$\circ$}}
\def\startailstar{\hbox{$*$}\kern-1.5pt\hbox{$-$}\kern-1.5pt\hbox{$*$}}
\def\startail{\hbox{$*$}\kern-1.5pt\hbox{$-$}}
\def\tailcirc{\hbox{$-$}\kern-1.5pt\hbox{$\circ$}}
\def\circtail{\hbox{$\circ$}\kern-1.5pt\hbox{$-$}}
\def\starrightarrow{\hbox{$*$}\kern-1.5pt\hbox{$\rightarrow$}}
\def\leftarrowstar{\hbox{$\leftarrow$}\kern-1.5pt\hbox{$*$}}

\icmltitlerunning{Causal Structure Discovery from Distributions Arising from Mixtures of DAGs}

\begin{document}

\twocolumn[
\icmltitle{Causal Structure Discovery 
\\from  Distributions Arising from
           Mixtures of DAGs}



\icmlsetsymbol{equal}{*}

\begin{icmlauthorlist}
\icmlauthor{Basil Saeed}{mit}
\icmlauthor{Snigdha Panigrahi}{umich}
\icmlauthor{Caroline Uhler}{mit,eth}
\end{icmlauthorlist}

\icmlaffiliation{mit}{Laboratory for Information and Decision Systems and Institute for Data, Systems and Society, Massachusetts Institute of Technology, Cambridge, MA, USA}
\icmlaffiliation{umich}{Department of Statistics, University of Michigan, Ann Arbor, MI, USA}
\icmlaffiliation{eth}{Department of Biosystems Science and Engineering,  ETH Zurich, Switzerland}

\icmlcorrespondingauthor{Caroline Uhler}{cuhler@mit.edu}


\vskip 0.3in
]

\printAffiliationsAndNotice{}

\begin{abstract}
We consider distributions arising from a mixture of causal models, where each model is represented by a directed acyclic graph (DAG). 
We provide a graphical representation of such mixture distributions and prove that 
this representation encodes the conditional independence relations of the mixture distribution. 
We then consider the problem of structure learning based on samples from such distributions. Since the mixing variable is latent, we consider causal structure discovery algorithms such as FCI that can deal with latent variables. We show that such algorithms recover a ``union'' of the component DAGs and can identify variables whose conditional distribution across the component DAGs vary. We demonstrate our results on synthetic and real data showing that the inferred graph identifies nodes that vary between the different mixture components.
As an immediate application,
we demonstrate how retrieval of this causal information can be used to cluster samples according to each mixture component. 
\end{abstract}

\section{INTRODUCTION}

\label{intro}

Determining causal structure from data is a central task in many applications. 
 \cite{friedman2000using,heckerman1995real}
Causal structure is often modeled using a \emph{directed acyclic graph} (\textrm{DAG}), where the nodes represent the variables of interest, and the directed edges represent the direct causal effects between these variables~\cite{pearl2009causality}. Assuming that the generating distribution of the data factors according to the DAG provides a way to relate the conditional independence relations in the distribution to separation statements in the DAG (known as \emph{d-separation}) through the \emph{Markov property}~\cite{lauritzen1996graphical}. When not all variables of interest can be measured, DAGs are not sufficient to represent the observed distribution, since latent variables may introduce confounding effects between the observed variables. 
Instead, a family of mixed graphs known as \emph{maximal ancestral graphs} (\textrm{MAGs}) can be used to model the observed variables by depicting the presence of latent confounders between pairs of variables through bidirected edges~\cite{richardson2002ancestral}.

With respect to learning the causal graph from data, the most ubiquitous methods infer d-separation relations by estimating conditional independence relations from the data; examples are the PC and GSP algorithms in the fully observed setting, and the FCI algorithm in the presence of latent variables~\cite{spirtes2000causation,solus2017consistency, zhang2008causal}. 
\mbox{These algorithms are consistent under} the \emph{faithfulness assumption}, which asserts that every conditional inde- pendence relation in the distribution corresponds to a d-separation relation in the graph. 
Note that even under faithfulness, the causal graph is in general not fully identifiable from observational data; it can in general only be identified up to its \emph{Markov equivalence class}~\cite{spirtes2000causation}.

In various applications, data used for causal structure discovery is \emph{heterogeneous} in that it stems from different causal models on the same set of variables~\citep{gates2012group,chu2003statistical,ramsey2011meta}. This is relevant for example in biomedical applications, where the goal is to learn a gene regulatory network based on gene expression data from a disease that consists of multiple not well characterized subtypes (as is the case for many neurological diseases). %
In such scenarios, the samples stem from a mixture of different causal models on the same set of variables, and the causal effects of the mixture distribution can in general not be faithfully represented by a single DAG. 

Furthermore, a single DAG inferred from such samples cannot identify differences between the component DAGs 
in the mixture, which may be critical for personalized biomedical interventions, and may lead to flawed conclusions downstream. 

In this work, we consider distributions arising as mixtures of causal DAGs. Our main contributions are as follows: 
\vspace{-0.4cm}
\begin{itemize}
\item We introduce the \emph{mixture graph} to represent such mixture distributions. We prove 
that this graph encodes the conditional independence relations in the mixture distribution through separation statements (Theorem~\ref{thm:markov-mixture}) and show that the separation statements in every such graph can be realized by independence relations in some mixture distribution
(Proposition~\ref{prop_realizability}).
\vspace{-0.1cm}
\item We introduce the \emph{union graph}, a graph defined from the mixture graph. We prove that, under a faithfulness and ordering assumption on the DAGs in the mixture, the FCI algorithm applied to data from a mixture of DAGs outputs the union graph (Theorem \ref{thm:mixture-union}).
\item We prove that the union graph can be used to identify variables whose conditional distribution across the component DAGs changes 
(Proposition~\ref{prop:varying}). 
We demonstrate the implication of this result for identifying critical nodes and for clustering samples according to their mixture component on synthetic data and data from genomics.

\end{itemize}

\section{PRELIMINARIES \& RELATED WORK} 
\label{prelims}

\subsection{Graphical representations: DAGs and MAGs}
In this paper, we consider two types of graphs: \emph{directed acyclic graphs} (\emph{DAGs}) and \emph{mixed graphs} with directed ($\rightarrow$) and bidirected ($\leftrightarrow$) edges. 
We denote the former by $\cD = (V,E)$ and the latter by $\cM = (V,D,B)$, where $V$ denotes the set of vertices, $E$ and $D$ denote the set of directed edges and $B$ denotes the set of bidirected edges.
A mixed graph is said to be \emph{ancestral} if it has no directed cycles, and whenever there is a bidirected edge $u\leftrightarrow v$, then there is no directed path from $u$ to $v$~\cite{richardson2002ancestral}. While ancestral graphs have been defined more generally to allow also for undirected edges, in this work we will only make use of graphs with directed and bidirected edges.

Throughout, we will use the notation $\ch_{\cM}(v)$, $\pa_{\cM}(v)$ and $\an_{\cM}(v)$ to denote the children, parents and ancestors, respectively, of a node $v$ in the graph $\cM$. Furthermore, we use the standard definitions of \emph{path} and \emph{directed path} in a graph; for these definitions, see e.g.~\citet{lauritzen1996graphical}. 
We will use the notation $v\leftrightarrow_\cM u$ as a shorthand to denote ``the edge $v\leftrightarrow u$ between nodes $u,v$ in $\cM$'', and use similar notations for other types of edges.

The notion of d-separation from DAGs can be generalized to ancestral graphs by accounting for the new possible ways to obtain a collider from bidirected edges \cite{richardson2002ancestral}.
In ancestral graphs, unlike in DAGs, it is possible to have a pair of nodes that are not adjacent, but cannot be d-separated given any subset of nodes. 
An ancestral graph where any non-adjacent pair of nodes is d-separated given some subset of nodes is called \emph{maximal}, and a non-maximal ancestral graph can be made maximal by adding a bidirected edge between all such pairs. 
An ancestral graph that is maximal is called a \emph{Maximal Ancestral Graph} (\textrm{MAG}) ~\cite{richardson2002ancestral}.

Ancestral graphs are a useful representation of DAGs with unobserved nodes. 

Specifically, \citet{richardson2002ancestral} showed that given a DAG $\cD = (V\cup L, E)$, with observed nodes $V$ and unobserved nodes $L$, satisfying a set of d-separation statements of the form ``$A$ d-separated from $B$ given $C$'' for disjoint $A,B,C\subseteq V$, there exists an ancestral graph $\cM  = (V, D, B)$ with the same d-separation statements, called the \emph{marginal ancestral graph of $\cD$ with respect to $L$}. ~\citet{sadeghi2013stable} gave a local criterion to construct this graph from $\cD$. Throughout our paper, we will make use of this in the special case where $L$ consists of a single node of in-degree $0$. The specialization of Sadeghi's algorithm to this case is provided in Algorithm~\ref{alg:marginalization}.

\begin{algorithm}[h]
\caption{Algorithm 1: Construction of the marginal ancestral graph}
 \label{alg:marginalization}
 \begin{spacing}{1.1}
 \begin{algorithmic} 
 \STATE \hspace{-4mm}\textbf{Input:} 
 DAG $\cD = (V \cup \{y\}, E)$, where $y$ has in-degree $0$.\\
 \hspace{-4mm}\textbf{Output:} the marginal ancestral graph of $\cD$ w.r.t. $y$.\\
 \textbf{(0)} Initialize $D = \emptyset$, $B = \emptyset$\\
 \textbf{(1)} For $u, v \in \ch_{\D}(y)$: add $u\leftrightarrow v$ to $B$.\\
 \textbf{(2)} For $t, u, v$ such that $(t \rightarrow u) \in E$ and $(u\leftrightarrow v) \in B$: \\
 \hspace{4mm} if $u\in \an_{\D}(v)$, then add $t \rightarrow v$ to $D$.\\
 \textbf{(3)} For $u, v$ such that $u\leftrightarrow v \in B$: if $u \in \an_{\D}(v)$, then\\
 \hspace{4mm} remove $u\leftrightarrow v$ from $B $ and add $u \rightarrow v$ to $D$.\\
 \textbf{(4)} Return the ancestral graph $\cM = (V, D, B)$.
 \end{algorithmic}
 \end{spacing}
\end{algorithm}

Although, in general, the ancestral graph constructed using Sadeghi's criterion is not maximal, the relevant restriction considered here, i.e.,~when $L$ consists of a single node with in-degree 0, is always a MAG. The following proposition states this; a proof is provided in section~\ref{section:proof_prop_marginal_mag} of the Appendix

\vspace{0.2cm}

\begin{prop}
\label{prop:marginal_mag}
The output of Algorithm~\ref{alg:marginalization} is a MAG.
\end{prop}

\subsection{Markov Properties}
Given a graph $\cM$ with nodes $V$, we associate to each node $v\in V$ a random variable $X_v$ and denote the joint distribution of $X_V := (x_v: v\in V)$ by $p_{X_V}$. The \emph{Markov property} associates missing edges in $\cM$ with conditional independence statements in $p_{X_V}$: a distribution $p_{X_V}$ is said to satisfy the Markov property with respect to $\cM$ if for any disjoint $A,B,C\subseteq V$ such that $A$ and $B$ are d-separated given $C$ in $\cM$, it holds that $X_A \ci X_B \mid X_C$ in~$p_{X_V}$~\cite{lauritzen1996graphical}. 
For DAGs, an equivalent condition to the Markov property 
 is for $p_{X_V}$ to factorize as 
 $p_{X_V}(x_V)=\prod_{v\in V} p(x_v| x_{\pa_{\cG}(v)});$
 see~\citet{lauritzen1996graphical}. Considering latent variables $X_L$,~\citet{richardson2002ancestral} showed that given a distribution $p_{X_{V},X_L}$ that is Markov with respect to a DAG $\cD$ over $V\cup L$, the marginal $p_{X_V}(x_V)=\sum_{x_L}p_{X_V, X_L}(x_V,x_L)$ is Markov with respect to the marginal ancestral graph of $\cD$ with respect to $L$.

It is possible for two different DAGs $\cD_1,\cD_2$ over the same set of nodes to satisfy the same set of d-separation statements. In this case, $\cD_1$ and $\cD_2$ are said to be \emph{Markov equivalent}, and the set of all DAGs that are Markov equivalent to a DAG $\cD$ is called the \emph{Markov equivalence Class of $\cD$}. These definitions trivially extend to MAGs. The Markov equivalence class of a MAG can be represented by a \emph{partial ancestral graph} (\emph{PAG}): the edges in such a graph have three types of tips: arrowheads ($\leftarrow$), tails $(-)$ and circles $\circtail$, where arrowhead (tail) signifies that this arrowhead exists in all graphs in the Markov equivalence class~\cite{zhang2008causal}.

\subsection{Causal Structure Discovery}
\label{section:background_faithfulness}
The goal of structure learning is to recover the graph $\cD$ or $\cM$ from data generated from the distribution $p_{X_V}$. This task often requires assumptions beyond the Markov property. One common such assumption is the so-called \emph{faithfulness assumption} which states that for any disjoint $A,B,C\subseteq V$, it holds that $A$ and $B$ are d-separated given $C$ whenever $X_A \ci X_B \mid X_C$ in $p_{X_V}$~\cite{spirtes2000causation}.
The faithfulness assumption allows making inference about the structure of $\cD$ or $\cM$ from conditional independence tests on the data. Various algorithms have been proposed for this task that are provably consistent, such as the PC, GES or GSP algorithms for learning DAGs~\cite{spirtes2000causation,chickering2002optimal,solus2017consistency}, and the FCI algorithm for learning MAGs~\cite{spirtes2000causation}. Note that even under the faithfulness assumption, it is in general only possible to retrieve the Markov equivalence class of a graph $\cD$ or $\cM$ from data; this is the output of the above algorithms. For example,  FCI in general does not return a specific MAG, but a PAG representing a Markov equivalence class of MAGs.

\subsection{Causal Inference from Mixtures of DAGs}
While the problem of learning appropriate representations from data of DAG mixtures arises in various applications, little work has been done on theory and methodology in this direction.
\citet{spirtes1994conditional} investigated the conditional independence properties of such mixture distributions; he defined a cyclic graphical model derivable from the component DAGs and proved that the mixture distribution is \mbox{Markov with respect to it. 
However, this} graph does not capture the full set of conditional independence relations for any reasonable mixture. In fact, as we discuss later, this graph is similar to the representation we define in Section~\ref{section:union}, which also only provides partial information about the structure of the component DAGs. 
To capture the full set of independences in the mixture distribution, a representation sparser than that of~\citet{spirtes1994conditional} is necessary. \citet{strobl2019global,pmlr-v104-strobl19a} built on this work to define a sparser graph. However, we provide examples in Section~\ref{section:strobl_example} of the Appendix showing that the Markov condition in general does not hold for this graph, i.e., there can be d-separation statements in the graph that do not correspond to conditional independence relations in the mixture distribution. 
Finally,~\citet{ramsey2011meta} provided conditions for the mixture distribution to be representable by a graph that is a union of the component DAGs.

To learn the component DAGs from mixture data, a simple approach is to cluster the data using, for example, the Expectation-Maximization (EM) algorithm and then learn a DAG from each cluster. This, however, uses a reduced sample size to learn each DAG (corresponding to the size of the associated cluster). 
In the case where the cluster labels are known and the DAGs are related, \citet{wang_JCI} showed that learning each DAG separately can lead to loss in accuracy compared to when the full sample size is used to learn the DAGs jointly. 
When the expectation in the EM algorithm can be computed, as e.g.~for Gaussians, ~\citet{thiesson2013learning} proposed a heuristic approach based on the EM algorithm to directly learn the component DAGs from the mixture data. In this work, we consider a different problem. Instead of learning the component DAGs we provide a graphical representation of the mixture distribution and identify critical aspects of the component DAGs that are captured by this graph and can be identified by algorithms such as FCI when applied directly to the mixture distribution. 

\section{MIXTURE DAG AND MARKOV PROPERTY}
\label{model}

In this section, we provide our first main result: after formally introducing distributions that arise as mixtures of DAGs, we define the \emph{mixture DAG} 
and prove in Theorem~\ref{thm:markov-mixture} and Proposition~\ref{prop_realizability} that it is a valid representation of the model, i.e., the DAG encodes the conditional independence relations of the mixture distributions. More precisely, not only is the Markov condition satisfied (i.e., all separation statements in the mixture DAG correspond to conditional independence relations in the mixture distribution), but in addition, every mixture DAG is also realizable by a mixture distribution (meaning that the mixture DAG cannot be made sparser without losing the Markov property).

\subsection{Mixture of Causal DAGs}
To introduce the mixture model, we consider $K$ DAGs $\{\D\up1,\dots,\D\up K\}$ with $\mathcal{D}\up j =(V, E\up j)$ for $1\le j\le K$, i.e., these $K$ DAGs are defined on the \emph{same} set of nodes. 

Associated with each component DAG $\mathcal{D}\up{j}$ is a random vector $X_V$ with  distribution $p\up{j}(x_V)$. Let $V_{\;\text{INV}}$ denote the set of nodes that are \emph{invariant} across the $K$ component DAGs, i.e., nodes whose conditional distribution in the factorization does not vary across $\cD\up1,\dots,\cD\up K$; that is 
\begin{equation}
\label{invariant:nodes}
\begin{aligned}
 V_{\;\text{INV}} \hspace{-1mm}=\hspace{-1mm} \Big\{\hspace{-0.5mm}v\hspace{-0.5mm}\in \hspace{-0.5mm}V\hspace{-1mm}:  p\up{j}&(x_v|x_{\pa_{\D\up{j}}(v)}) =p\up{k}(x_v|x_{\pa_{\D\up{k}}(v)})\\
&\textrm{ for all }j,k\in \{1,2,\cdots, K\} \Big\}.
\end{aligned}
\end{equation} 

Assuming that each distribution $p\up{j}(x_V)$ admits a factorization according to DAG $\cD\up j$, we then obtain:
\begin{align*}
p\up{j}(x_V) &=\hspace{-4.5mm} \prod_{v\in V\setminus V_{\;\text{INV}}}\hspace{-4mm}p\up{j} (x_v\lvert x_{\pa_{\cD\up{j}}(v)})\hspace{-2mm}\prod_{v\in V_{\;\text{INV}}}\hspace{-2mm} p\up{j} (x_v\lvert x_{\pa_{\cD\up{j}}(v)})\\
&=\hspace{-4.5mm} \prod_{v\in V\setminus V_{\;\text{INV}}}\hspace{-4mm}p\up{j} (x_v\lvert x_{\pa_{\cD\up{j}}(v)})\hspace{-2mm}\prod_{v\in V_{\;\text{INV}}}\hspace{-2mm} p\up{1} (x_v\lvert x_{\pa_{\cD\up{1}}(v)})
\end{align*}
for all $1\le j\le K$, i.e., each distribution 
decouples into two components: one over the variables associated with $V_{\;\text{INV}}$ that remains constant across all $K$ distributions, and another over the remaining variables which may differ with $j$. 

Let $J$ be a discrete variable taking values in $\{1,\dots,K\}$ with probabilities $p_J(j)$ for each $j\in \{1,\dots,K\}$. 
Defining a joint distribution $p_\mu$ over $X_V \cup J$ by
\begin{equation}
\label{mixture}
p_\mu(x_V,j):=
p_J(j)\cdot p\up j(x_V),
\end{equation}
this joint distribution satisfies $p\up j(x_V)= p_\mu(x_V |J~=~j)$ and the observed mixture distribution is obtained by marginalizing $p_\mu$ over the unobserved index variable $J$. With a slight abuse of notation, we denote the resulting mixture distribution also by $p_\mu$. Given samples from this distribution, i.e., without knowledge of the membership of each sample to its generating DAG, we analyze what can still be inferred regarding the structure of $\cD\up1,\dots,\cD\up K$.

\subsection{Mixture DAG and Markov Property}
\label{section:mixture_graph}

We now present the \emph{mixture DAG}, a DAG that is representative of the independence relations induced amongst the observed variables after marginalizing over the index variable $J$ in (\ref{mixture}). Denoting the number of vertices in $V$ by $|V|$, the mixture DAG is a graph on $K\cdot|V|+1$ nodes constructed by placing the $K$ component DAGs next to each other, giving rise to a DAG on $K\cdot|V|$ nodes, and using an additional node to represent $J$. We now provide the precise definition.

\begin{defn}[Mixture DAG]
\label{mixture:DAG:defn}
Let $v^{(j)}$ denote vertex $v$ in DAG $j$ and let $[V]:=\cup_{1\leq j\leq K} V\up j$ denote the vertices of the $K$ component DAGs. The \emph{mixture DAG}, denoted by $\D_\mu$, has nodes $[V] \cup \{y\}$ and edges $E_\mu$ consisting of edges in each component DAG, namely
$$
\vspace{-1.2mm}
\bigcup_{j=1}^K \Big\{v\up j \rightarrow \tilde v\up j: v,\tilde v\in V,\; v\rightarrow \tilde v \in E\up j\Big\},
\vspace{-0.2mm}
$$
and additional edges from node $y$ to some nodes in $[V]$, namely those corresponding to variables that have conditionals that are not the same for all $j$, i.e.,
$$
\bigcup_{j=1}^K\Big\{y \rightarrow v\up j : v\in V\setminus V_{\textrm{INV}}\}.
\vspace{-1.2mm}
$$
\end{defn}
Figure~\ref{fig:ex1} 
provides an example of the mixture DAG arising from a mixture with $K=2$ and $|V|=4$. Note that, while the results of this section hold even when 

the DAGs $\cD\up j$ 
have no common topological ordering (meaning that there exists no ordering $\pi$ such that $v< u$ in $\pi$ only if $u\not\in\an_{\cD\up j}(v)$ for all $1\leq j\leq K$), the mixture DAG is sparsest, and hence provides information about the component DAGs through separation statements, when a common topological ordering exists (as in Figure~\ref{fig:ex1}). When there is no common ordering, the set $V_{\textrm{INV}}$ is generally smaller, since $\pa_{\cD\up j}\neq\pa_{\cD\up k}$ implies $p\up{j}(x_v|x_{\pa_{\D\up{j}}(v)}) \neq p\up{k}(x_v|x_{\pa_{\D\up{k}}(v)})$, which implies a denser mixture DAG.

\begin{figure*}[t!]
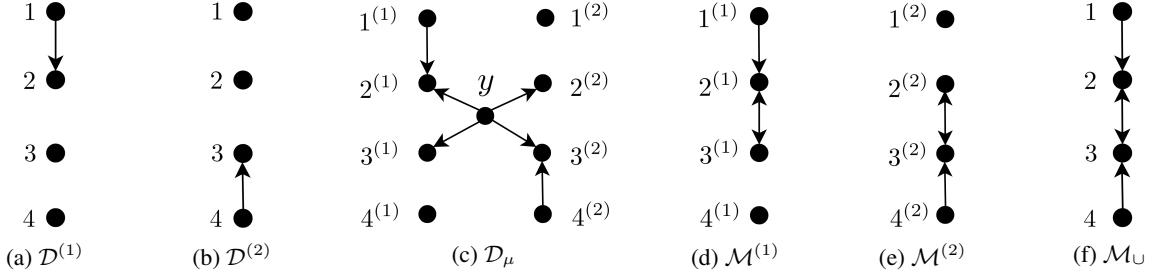

  \centering
  \begin{subfigure}[t]{0.14\textwidth}
  \centering
    \includegraphics[height=30mm]{figures/ex1_D1.png}
    \caption{$\cD\up 1$}
  \end{subfigure}
  \begin{subfigure}[t]{0.14\textwidth}
  \centering
    \includegraphics[height=30mm]{figures/ex1_D2.png}
    \caption{$\cD\up 2$}
  \end{subfigure}
  \begin{subfigure}[t]{0.24\textwidth}
  \centering
    \includegraphics[height=30mm]{figures/ex1_mixture.png}
    \caption{$\cD_\mu$}
  \end{subfigure}
  \begin{subfigure}[t]{0.14\textwidth}
  \centering
    \includegraphics[height=30mm]{figures/ex1_M1.png}
    \caption{$\cM\up 1$}
    \label{fig:ex1_M1}
  \end{subfigure}
  \begin{subfigure}[t]{0.14\textwidth}
  \centering
    \includegraphics[height=30mm]{figures/ex1_M2.png}
    \caption{$\cM\up 2$}
    \label{fig:ex1_M2}
  \end{subfigure}
  \begin{subfigure}[t]{0.14\textwidth}
  \centering
    \includegraphics[height=30mm]{figures/ex1_union.png}
    \caption{$\cM_\cup$}
    \label{fig:ex1_union}
  \end{subfigure}
  \caption{(a)-(b): component DAGs for a mixture model with $K=2$; (c): corresponding mixture DAG (see Definition~\ref{mixture:DAG:defn}); (d)-(e): associated component MAGs (see Section~\ref{section:union}); (f): associated union graph (see Definition~\ref{def:union_graph}).}
  \label{fig:ex1}
 \end{figure*}

We emphasize here that the DAG in Definition \ref{mixture:DAG:defn} is not a graphical model representation of the mixture distribution in the standard sense. This is already clear from the fact that the mixture DAG has $K\cdot|V|+1$ nodes, whereas the mixture distribution is only $|V|$-dimensional.  
Yet, in the following theorem we show that it is possible to read off  conditional independence relations that hold in the mixture distribution $p_\mu$ from the mixture graph in an intuitive manner. 

For $A\subset V$, we use the notation $[A]$ to denote all $K$ copies of the nodes in $A$, i.e., $A=\cup_{1\leq j\leq K} A\up j$.

\begin{thm}[Markov Property]
\label{thm:markov-mixture}

Let $A,B,C \subseteq V$ be disjoint. If $[A]$ and $[B]$ are d-separated given $[C]$ in the mixture DAG $\D_\mu$, then $X_A\!\!\ci\!\! X_B | X_C$ in the mixture distribution~$p_\mu$. 

\end{thm}

To illustrate this result, consider the example in Figure~\ref{fig:ex1_mixture_mag}. Since $[1] = \{1\up 1,1\up 2\}$ and $[4] = \{4\up 2, 4\up 2\}$ are d-separated given $\emptyset$ in the mixture DAG, then the mixture distribution $p_\mu(x_{1},x_{2},x_{3},x_{4})$ satisfies $X_{1}\ci X_{4}$.

We note that while the graphical representation provided by~\citet{pmlr-v104-strobl19a} (the \emph{mother graph}) is similar to the mixture DAG, it critically differs in how the component DAGs are connected via the node $y$. Importantly, we show in Section~\ref{section:strobl_example} in the Appendix that the mixture distribution $p_\mu$ is \emph{not} Markov with respect to the mother graph\footnote{\citet{strobl2019global,pmlr-v104-strobl19a} provides two different constructions; we show that the Markov property does not hold in either.}.

In the following, we provide a proof for Theorem \ref{thm:markov-mixture}. 
For each $1\leq j\leq K$, let $\widetilde\cD^{(j)}$ be the sub-DAG induced by $\cD_\mu$ on the vertices $V^{(j)}\cup\{y\}$.The main ingredient of the proof is the following lemma, which connects d-separation statements in the mixture DAG to conditional independence relations in the mixture distribution via d-separation in  $\widetilde{D}\up j$.

\begin{lemma}
\label{lemma:technical}
Let $A,B,C\subseteq V$ be disjoint. If for all $1\le j\le K$ it holds that
\begin{enumerate}
\vspace{-3mm}
    \item[(a)] $A\up j$ and $B\up j$ are d-separated given $C\up j$, and; 
    \vspace{-1mm}
    \item[(b)] $A\up j$ and $y$ are d-separated given $C\up j$ in $\widetilde\cD\up j$,
\vspace{-3mm}
\end{enumerate}
then $X_A\ci J\mid X_C$ in $p_\mu$, implying the factorization 
$$
p\up j(x_A,x_B|x_C) = p\up 1(x_A|x_C)p\up j(x_B|x_C)
$$
for all $1\le j\le K$.
\end{lemma}

We now provide the proof for Theorem~\ref{thm:markov-mixture}.
\vspace{-0.1cm}
\begin{proof}[Proof of Theorem~\ref{thm:markov-mixture}]
\vspace{-2mm}
We start by showing that the conditions of Lemma~\ref{lemma:technical} are satisfied.
First, note that $[A]$ and $[B]$ are d-separated given $[C]$ in $\D_\mu$ implies that $A^{(j)}$ and $B^{(j)}$ are d-separated given  $C^{(j)}$ in $\D^{(j)}$ for all $1\le j \le K$. Second, note that since $y$ has in-degree $0$, we cannot have both a d-connecting path given $[C]$ between $[A]$ and $y$ and one between $[B]$ and $y$ in $\D_\mu$. Hence, we may assume without loss of generality that $[A]$ and $y$ are d-separated given $[C]$ (otherwise, $[B]$ and $y$ are d-separated given $[C]$). 

We now use Lemma~\ref{lemma:technical} to show that $p_\mu(x_A,x_B|x_C)$ factorizes as
$f_A(x_A,x_C) f_B(x_B,x_C)$, which would prove that $X_A\ci X_B | X_C$ in $p_\mu$. By definition of $p_\mu$ in \eqref{mixture}, 
\vspace{-0.1cm}
$$
\vspace{-2mm}
p_\mu(x_A,x_B|x_C) = \sum_{j=1}^K p^{(j)}(x_A,x_B|x_C)p_J(j),
$$
and hence as a consequence of Lemma~\ref{lemma:technical} we obtain
\begin{align*}
p_\mu(x_A,x_B|x_C) &= \sum_{j=1}^K p\up 1(x_A|x_C)p^{(j)}(x_B|x_C)p_J(j) \nonumber\\
& = p^{(1)}(x_A|x_C) \sum_{j=1}^K p^{(j)}(x_B|x_C)p_J(j), \nonumber
\end{align*}
providing a factorization of the desired form.

\end{proof}

In Theorem~\ref{thm:markov-mixture}, we established that every separation statement in the mixture DAG $\cD_\mu$ corresponds to a conditional independence relation in the mixture distribution $p_\mu$. Next, we show that every mixture DAG is~\emph{realizable}, i.e., that for any mixture DAG $\cD_\mu$, there exists a $p_\mu$ whose conditional independence relations are faithfully represented by the separation statements of $\cD_\mu$. This implies that $\cD_\mu$ is the ``correct'' graphical representation of a mixture of DAGs and cannot be made sparser without losing the Markov property.

\subsection{Faithfulness}

We define faithfulness of a mixture distribution $p_\mu$ with respect to a mixture DAG $\cD_\mu$ analogously to how faithfulness is defined for a distribution with respect to a DAG model.

\begin{defn}[Mixture Faithfulness]
\label{faithfulness:mixture}
The mixture distribution $p_\mu$ is faithful with respect to a mixture DAG $\cD_\mu$ if for any disjoint $A,B,C \subseteq V$ with $X_A \ci X_B \lvert X_C $ in $p_\mu$ it holds that $[A]$ and $[B]$ are d-separated given $[C]$.
\end{defn}

We next provide an example showing that mixture faithfulness is not implied by faithfulness of each component distribution $p^{(j)}$ with respect to the corresponding DAG $\mathcal{D}^{(j)}$. Hence, to establish realizability of the mixture graph, it is not sufficient to rely on the fact that for every DAG $\cD\up j$, there exists a distribution $p\up j$ that is faithful to it.

\begin{Example} 
\label{example:faithfulness}
Consider the distributions $p^{(1)}(x_V), p\up2(x_V)$ on $V = \{1,2,3,4\}$ that factor according to the DAGs $\cD\up1,\cD\up2$, respectively, shown in Figure~\ref{fig:ex1}. Namely
\begin{align*}
p\up 1(x_V) = p\up1(x_1)p\up1(x_2|x_1)p\up 1(x_3)p\up1(x_4),\\
p\up 2(x_V) = p\up2(x_1)p\up2(x_2)p\up 2(x_3|x_4)p\up2(x_4),
\end{align*}
where
\begin{align*}
&p\up 1(x_1) = \cN(x_1; 0,1), \hspace{10mm} p\up 2(x_1) = \cN(x_1; 0,1),\\
&p\up 1(x_2 | x_1) = \cN(x_2; x_1, 1 ), \hspace{4mm} p\up2(x_2)= \cN(x_2; 0,2),\\
&p\up 1(x_3) = \cN(x_3; 0, 1), \hspace{4mm} p\up2(x_3|x_4)= \cN(x_3; x_4, 1),\\
&p\up 1(x_4) =  \cN(x_4; 0,1), \hspace{11mm} p\up 2(x_4) = \cN(x_4; 0,1).
\end{align*}
Then, defining  
$p_\mu(x_V) := \sum_{j=1}^2 p\up j(x_V)p_J(j),$
for some $J\sim p_J(j)$, we obtain that
\begin{align*}
p_\mu&(x_2,x_3) = \int p_\mu(x_V) dx_1dx_4\\
           &=\int\hspace{-1mm} p_J(1)p\up1(x_1)p\up 1(x_2|x_1)p\up1(x_3)p\up1(x_4)dx_1dx_4\\
           &\hspace{2mm}+\hspace{-1.5mm}\int\hspace{-1mm} p_J(2)p\up2(x_1)p\up2(x_2)p\up2(x_3|x_4)p\up2(x_4)dx_1dx_4\\
           &=p_J(1) \;\cN(x_2 ;0,2)\; \cN(x_3; 0,1)\\
           &\hspace{2mm}+p_J(2) \;\cN(x_2 ;0,2)\; \cN(x_3; 0,2)\\
           &=\cN(x_2; 0,2)\Big(p_J(1)\cN(x_3;0,1) + p_J(2)\cN(x_3;0,2) \Big)\\
           &= f(x_2) g(x_3),
\end{align*}
which implies that $X_2 \ci X_3$ in $p_\mu$, although in the mixture DAG corresponding to $p_\mu$ shown in Figure~\ref{fig:ex1} the nodes $2$ and $3$ are d-connected via the path through $y$. \hfill\qed
\end{Example}

This example was carefully crafted; even a slight perturbation such as choosing $p\up 2(x_2) = \cN(x_2; 0, 2.001)$ would have meant that $p_\mu(x_2,x_3)$ does not factor, 
indicating that mixture-faithfulness violations are rare. More precisely, consider the family of Gaussian mixture models where each $p\up j$ is a Gaussian distribution that is faithful with respect to $\cD\up j$. A violation of mixture-faithfulness occurs if and only if 
$\sum_{j}p\up j(x_A,x_B|x_C)$ factors as $ p_\mu(x_A|x_C)p_{\mu}(x_B|x_C)$,~i.e., 
$$\sum_{j}p\up j(x_A,x_B|x_C)=
\sum_{i}p\up i(x_A|x_C)\sum_j p\up j(x_B|x_C),
$$
when $[A]$ and $[B]$ are d-connected given $[C]$ in $\cD_\mu$.
This represents an equality constraint on the parameters of the Gaussians $p\up j$ for $1\le j\le K$. 
As a consequence, mixture-faithfulness holds almost surely and any $\cD_\mu$ is realizable by a mixture of Gaussians, thereby proving the following.

\begin{prop}[Realizability of $\cD_\mu$]
\label{prop_realizability}
For any mixture DAG $\cD_\mu$, there exists a mixture distribution $p_\mu$ that is faithful with respect to $\cD_\mu$.
\end{prop}

\section{LEARNING FROM MIXTURE DATA}
\label{section:union}

Without knowing the membership of each sample to a component DAG, we cannot generally learn the structure of $\cD\up j$ for each $j$ from the data. Since the mixing variable is latent, an intuitive approach is to apply FCI to learn a MAG representation of $p_\mu$. In this section, we will characterize the output of FCI. In particular, we will show that FCI identifies critical nodes in the component DAGs: those whose conditionals across the component DAGs vary.

A difficulty for structure discovery using MAG-based learning algorithms such as FCI, is that even under the mixture-faithfulness assumption the conditional independence relations in a mixture distribution $p_\mu$ may not be representable by any MAG. We illustrate this in the following example and then provide conditions to avoid this phenomenon.

\vspace{0.1cm}
\begin{Example}
\label{ex:not_mag}


Consider $\cD_\mu$ shown in Figure~\ref{fig:not_mag_mixture}. We show that there does not exist any MAG $\widetilde\cM$ over the variables $V=\{1,\dots,5\}$ that satisfies: $A$ d-sep from $B$ given $C$ in $\widetilde\cM$ if and only if $[A]$ d-sep from $[B]$ given $[C]$ in $\cD_\mu$. First, note that such a MAG would need to have the same skeleton as the graph in Figure~\ref{fig:not_mag_union} to respect the adjacencies in $\cM_\mu$. Otherwise it would have an extra or missing d-separation with no analog in $\cM_\mu$. In addition, $\widetilde\cM$ would also need to contain the colliders $4\rightarrow 5 \leftarrow 2$ and $1 \rightarrow 2 \leftarrow 5$ to respect the d-separation relations resulting from  $4\up 2\rightarrow 5\up 2 \leftarrow y \rightarrow 2\up2$ and $1\up1 \rightarrow 2\up 1 \leftarrow y \rightarrow 5\up 1$ respectively. 
This implies the existence of $2\leftrightarrow_{\widetilde\cM} 5$. 
Further note that conditioning on either $[2], [3]$ or $[4]$ (or any subset of these) connects $[5]$ and $[1]$ in $\cD_\mu$ which are d-separated given $\emptyset$. The only orientation of arrowheads compatible with both the skeleton and these separation/connection relations is $2\rightarrow 3\rightarrow 4$. Hence, $4\in \de_{\widetilde \cM}(2)$. Finally, the existence of an arrowhead $4\leftarrowstar 5$ would violate the separation: $[5]$ d-separated from $[1]$ given $\emptyset$.
Hence, $2\leftrightarrow_{\widetilde\cM} 5$ and $2\in\an_{\widetilde\cM}(5)$, violating the ancestral property. \hfill\qed
\end{Example}

We now identify a class of mixture models for which the d-separations in the mixture DAG are equivalent to d-separation statements in a MAG. 

\begin{defn}
\label{def:ordering}
Let $\cM\up j$ be the MAG constructed via Algorithm~\ref{alg:marginalization} from the induced sub-DAG $\widetilde\cD^{(j)}$ defined in Section~\ref{section:mixture_graph}. 
The MAGs $\cM\up 1,\dots,\cM\up K$ are said to be \emph{compatible with the same poset} if there exists a partial order $\pi$ on $V$ such that for all $1\leq j\leq K$ it holds that
   \emph{(a)} $u\in\an_{\cM\up j}(v) \Rightarrow u<_\pi v$; and
    \emph{(b)} $u\leftrightarrow_{\M\up j} v \Rightarrow u \incomp_\pi v.$
\end{defn}

Figures~\ref{fig:ex1_M1} and~\ref{fig:ex1_M2} show examples of MAGs $\cM\up j$ that satisfy this poset compatibility condition. One can further check that the MAGs $\cM\up1$ and $\cM\up 2$ associated with the mixture DAG in Figure~\ref{fig:not_mag_mixture} do not satisfy this condition. 
This example shows that there exist DAGs $\cD\up1,\dots,\cD\up K$ with a common topological ordering whose corresponding MAGs $\cM\up 1,\dots,\cM\up K$ do not satisfy the poset compatibility condition~\ref{def:ordering}. On the other hand, it can be readily verified that the compatibility assumption on $\cM\up1,\dots,\cM\up K$ implies that $\cD\up1,\dots,\cD\up K$ have a common topological ordering.

In the following, we show that poset compatibility ensures that d-separation relations in $\cD_\mu$ are representable by a MAG, which we call the \emph{union graph} since it is obtained as a union of the edges of
$\cM\up 1 \dots,\cM\up K$.

\begin{defn}[Union Graph]
\label{def:union_graph}
The \emph{union graph} $\M_\cup:= (V, D_\cup,B_\cup)$ has vertices $V$, directed edges
   \begin{equation*}
   \begin{aligned}
    D_\cup = \{& v\rightarrow u : u, v \in V,\; \exists_{j} v^j \rightarrow_{\M\up j} u^j\},
    \end{aligned}
    \end{equation*}
    and bidirected edges
   $$B_\cup = \{v\leftrightarrow u : v,u\in V, \; \exists_{j} v^j \leftrightarrow_{\M \up j} u^j\}.$$
\end{defn}

We remark that~\citet{spirtes1994conditional} studied a similar graph and proved the Markov property for a DAG with vertices 
$V \cup \{y\}$ and directed edges given by the union of $\cD\up 1,\dots,\cD\up K$. 

An example of a union graph $\cM_\cup$ is given in Figure~\ref{fig:ex1_union}. In general, $\cM_\cup$ may neither be maximal nor ancestral (see Figure~\ref{fig:not_mag_union} for an example). However, the following lemma states that under poset compatibility it is guaranteed to be both. The proof is given in Section~\ref{section:proof_union_maximal_ancestral} of the Appendix, 
\begin{lemma}
\label{lemma:union_maximal_ancestral}
Under the assumption that $\cM\up 1 \dots,\cM\up K$ are compatible with the same poset, $\cM_\cup$ is a MAG.
\end{lemma}


We now state the main results of this section, characterizing the output of FCI when run on mixtures of DAGs. 

\begin{thm}
\label{thm:mixture-union}
Let $A,B,C \subseteq V$ be disjoint. If the component MAGs satisfy the poset compatibility assumption, then 
$A$ and $B$ are d-separated given $C$ in $\cM_\cup$ if and only if 
$[A]$ and $[B]$ are d-separated given $[C]$ in $\cD_\mu$.
\end{thm}

The proof is provided in Section~\ref{section:proof_mixture-union} in the Appendix. The following corollary follows directly from the asymptotic consistency of FCI~\cite{spirtes2000causation}.

\begin{cor}
\label{cor:fci}
    If the distribution $p_\mu$ is faithful with respect to a mixture DAG whose component MAGs satisfy the poset compatibility assumption~\ref{def:ordering}, then FCI outputs the Markov equivalence class of the corresponding union MAG $\cM_\cup$.
\end{cor}

We end this section by pointing out an important structural property of $\cM_\cup$, which can be used to recover key information about the component distributions in the mixture.
We leave the proof to Section~\ref{section:appendix_varying_proof} of the Appendix.
\begin{prop}
\label{prop:varying}
A bidirected edge $u\leftrightarrow v$ in the union graph $\cM_\cup$ implies that $u\in V\setminus V_{\textrm{INV}}$. Additionally, this implies that
$p\up j(x_u| x_{\pa_{\cD\up j}(u)})\neq p\up i(x_u | x_{\pa_{\cD\up i}(u)})$.
\end{prop}

Hence bidirected edges identify nodes in the component DAGs whose conditional distribution varies across mixture components. 
As we show in the following section, these nodes are natural candidates for features when clustering. 

\section{EXPERIMENTS}
\label{section:experiments}
\begin{figure*}[t!]
  \centering
  \begin{subfigure}[t]{0.246\linewidth}
  \centering
    \includegraphics[width=0.8\linewidth]{figures/no_mag_mixture.png}
    \caption{}
    \label{fig:not_mag_mixture}
  \end{subfigure}
  \begin{subfigure}[t]{0.246\linewidth}
  \centering
    \includegraphics[width=0.8\linewidth]{figures/no_mag_union.png}
    \caption{}
    \label{fig:not_mag_union}
  \end{subfigure}
  \begin{subfigure}[t]{0.246\textwidth}
    \centering
    \includegraphics[width=0.9\linewidth]{figures/apoptosis_fci.png}
    \caption{}
  \label{fig:apoptosis_fci}
  \end{subfigure}
  \begin{subfigure}[t]{0.246\textwidth}
    \centering
    \includegraphics[width=0.9\linewidth]{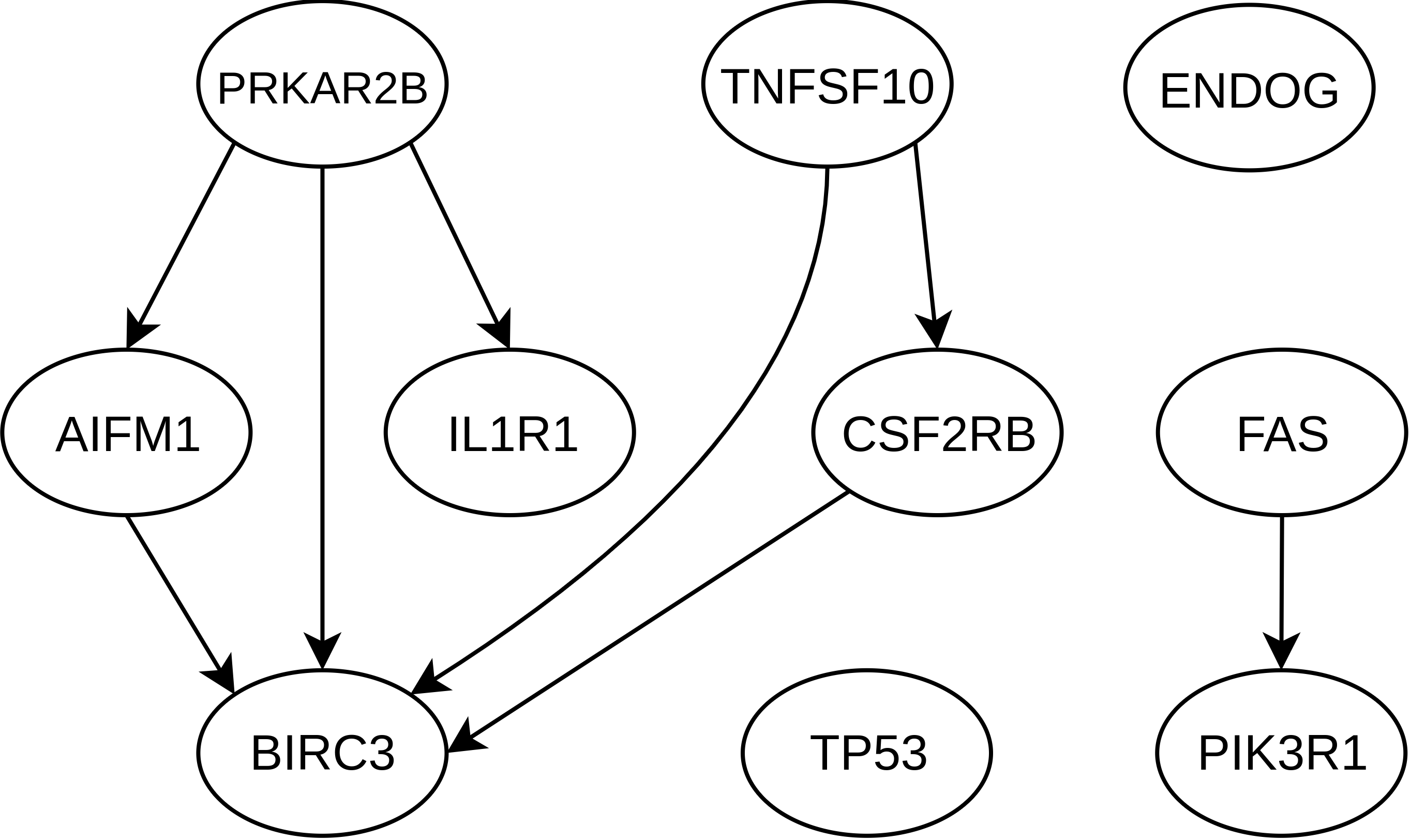}
    \caption{}
  \label{fig:apoptosis_dci}
  \end{subfigure}
  \\ \vspace{0.3cm}
    \begin{subfigure}[t]{0.246\textwidth}
    \includegraphics[width=\linewidth]{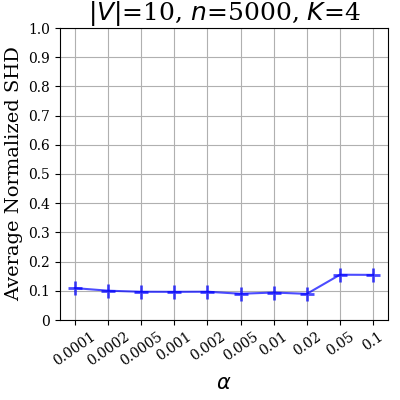}
  \caption{}
  \label{fig:shd}
  \end{subfigure}
  \begin{subfigure}[t]{0.246\textwidth}
    \includegraphics[width=\linewidth]{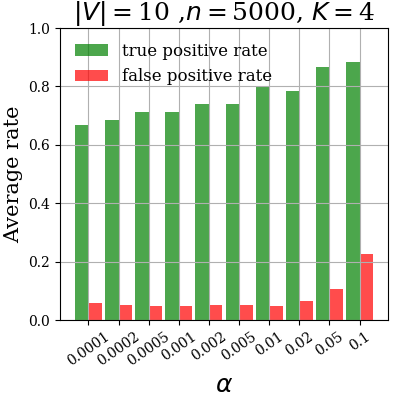}
    \caption{}
  \label{fig:varying}
  \end{subfigure}
  \begin{subfigure}[t]{0.246\textwidth}
    \includegraphics[width=\linewidth]{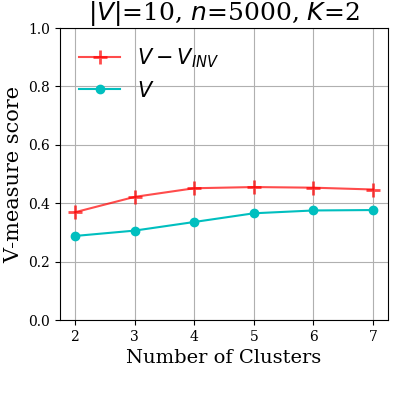}
    \caption{}
    \label{fig:clusteringa}
  \end{subfigure}
  \begin{subfigure}[t]{0.246\textwidth}
    \includegraphics[width=\linewidth]{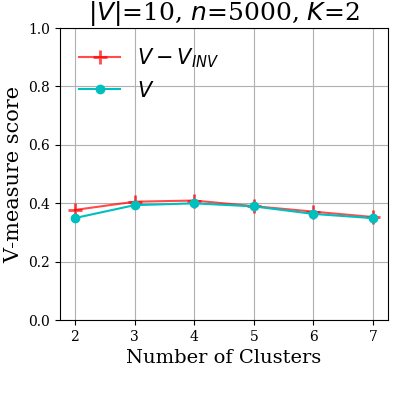}
    \caption{}
    \label{fig:clusteringb}
  \end{subfigure}
  \caption{
  (a) shows a mixture DAG $\cD_\mu$ and (b) shows the associated union graph $\cM_\cup$. In this model, each DAG $\cD\up j$ has a common topological ordering, however, the union MAG is not ancestral;
  (c) shows the output of FCI  on genes in the apoptosis pathway using mixture data without knowledge of the cluster membership for each sample, while (d) shows the difference graph of~\citet{wang2018direct} on the same genes learned when cluster membership of each sample is known;
    (e) shows the average normalized SHD between the PAG $\widehat\cP_\cup$ estimated using the mixture data, and $\widetilde\cP_\cup$ estimated using data sampled from $\cM_\cup$;
  (f) shows the true and false positive rate in estimating $V\setminus V_{\textrm{INV}}$;
  (g) shows the performance of clustering when the set $[V\setminus V_{\textrm{INV}}]$ has no descendants in $\cD_\mu$, while 
  (h) shows the same plot when $[V\setminus V_{\textrm{INV}}]$ has descendants in $\cD_\mu$. 
  }
  \label{fig:TBA}
\end{figure*}


\subsection{Synthetic Data}
In the following, we demonstrate the effectiveness of learning the union graph from mixture data, analyze the performance when estimating $V\setminus V_{\textrm INV}$ using Proposition~\ref{prop:varying}, and investigate the performance of clustering using mixture data when $V\setminus V_{\textrm INV}$ are used as features.

We generated $K$ component DAGs each with $|V| = 10$ nodes and the same topological ordering from an Erd\"os-R\'enyi model with expected degree $d=1.5/K$ so that the nodes in the $\cM_\cup$ have expected degree less than $1.5$.
From these DAGs, the corresponding MAGs $\cM\up j$ were computed using Algorithm~\ref{alg:marginalization}. If the MAGs were not compatible with the same poset, the DAGs were discarded to ensure poset-compatibility ($2$ out of $270$ graphs were discarded).

Data was sampled from each DAG based on a linear structural equation model with additive  Gaussian noise, where each edge weight $(u,v)$ was sampled uniformly in $[-2,-0.25] \cup [0.25, 2]$ (to ensure that it was bounded away from zero) and set to be equal for the edges $(u\up j,v\up j)$ for all $1\le j\le K$ if this edge existed in DAG $\cD\up j$. In this case, $v\in V_{INV}$ if and only if the parents of $X_v$ are the same across all $K$ DAGs. The mean for the Gaussian noise was sampled uniformly in $[-2,2]$ with standard deviation 1. From each DAG $\cD\up j$, we generated $n\, p_j$ observations where $\sum_{j=1}^K p_j = 1$ yielding a total of $n$ samples. For the plots in the main paper, we chose $p_j = 1/K$. We present additional plots in Appendix~\ref{appendix:experiments} for when $p = (p_k : 1\le k\le K)$ is sampled from a Dirichlet distribution.

\textbf{Learning the Union MAG}.
To evaluate Corollary~\ref{cor:fci}, we ran the R implementation of FCI from the \verb|pcalg| library on this synthetic data using Gaussian conditional independence tests (despite the true distribution being a mixture of Gaussians) with threshold $\alpha$. The output is a PAG $\widehat\cP_{\cup}$ representing the Markov equivalence class of the union graph. 
As comparison, we computed the true union graph $\cM_\cup$ based on the MAGs $\cM\up j$, generated $n$ samples from this graph (using a structural equation model with the same  parameters as in the mixture) and ran FCI on these samples to obtain an estimate $\widetilde\cP_\cup$ for the PAG of the union graph.
This offsets the estimation errors that are intrinsic to FCI. 
The difference between the PAGs $\widehat\cP_\cup$ and $\widetilde\cP_\cup$ was measured via a \emph{normalized structural Hamming distance};
the structural Hamming distance (SHD) between PAGs counts the occurrences of $\starrightarrow$ in one of the PAGs versus $\startail$ in the other, plus the number of adjacencies present in one graph but not the other.
The normalization is done by dividing over the possible number of errors for the realization at hand to keep the value in $[0,1]$ and make the numbers comparable.
Figure~\ref{fig:shd} shows the normalized SHD averaged over 30 realizations of synthetic datasets. We used $K=4$ and $n=5000$ in this plot; in Section~\ref{appendix:experiments} in the Appendix, we provide plots for $K\in\{2,6\}$ and $n\in\{1000,10000\}$.

\textbf{Identifying Nodes in $V\setminus V_{\textrm{INV}}$.}
To evaluate Proposition~\ref{prop:varying}, we estimated $ V\setminus V_{\textrm{INV}}$ by determining all nodes incident to bidirected edges in the PAG $\widehat\cP_\cup$ estimated using FCI. This set was compared to the ground truth; Figure~\ref{fig:varying} shows true positive and false positive rates for varying significance levels\footnote{We do not use ROC plots since while increasing the threshold  monotonically increases the true positive rate of the estimated adjacencies, it generally does not monotonically increase the number of correctly inferred edge orientations.},
averaged over $30$ realizations. We used $K=4$ and $n=5000$ in this plot. In Section~\ref{appendix:experiments} in the Appendix, we show plots for $K\in\{2,6\}$ and $n\in\{1000,10000\}$. 

\textbf{Clustering.}
Under mixture-faithfulness, $X_{V\setminus V_{\textrm{INV}}}$ represents the set of nodes whose conditionals vary across the component DAGs. 
This motivates using the nodes $X_{V\setminus V_{\textrm{INV}}}$ and their descendents 
as features for clustering since these are the only nodes with different marginals across the mixture components. 
Since FCI generally cannot identify all the descendents of $X_{V\setminus V_{\textrm{INV}}}$, we used only $X_{V\setminus V_{\textrm{INV}}}$ for clustering.  
As a proof-of-concept demonstrating that these features can be useful, we considered two settings, one in which $[V\setminus V_{\textrm{INV}}]$ has no descendants in $\cD_\mu$ (see Figure~\ref{fig:clusteringa}), and another one in which this set has descendants (Figure~\ref{fig:clusteringb}). 

In both settings, we used $\widetilde K$\emph{-means} clustering  for various values of $\tilde{K}$. 
To compare the quality of clustering using $[V\setminus V_{\textrm{INV}}]$ versus all nodes as features, we used the V-measure score from~\citet{rosenberg2007v} which is based on ground truth cluster assignments;
a higher score represents better performance.
As per what is expected from our theoretical results,  Figure~\ref{fig:clusteringa} shows that clustering based on the reduced number of features $[V\setminus V_{\textrm{INV}}]$ results in higher quality clusters as compared to using all features for clustering in the setting where 
$[V\setminus V_{\textrm{INV}}]$ has no descendants in $\cD_\mu$, while otherwise both feature sets perform equally.

\subsection{Real Data}
\label{section:real_data}

\textbf{Ovarian Cancer.} We applied this framework to gene expression data from ovarian cancer in $K=2$ patient groups (with 93 and 168 observations, respectively) with different survival rates~\cite{tothill2008novel}. We followed the analysis of~\citet{wang2018direct}, where the difference-DAG was estimated for the two groups based on the apoptosis pathway consisting of $|V|=10$ genes. The resulting difference-DAG is shown in Figure~\ref{fig:apoptosis_dci}. While the difference-DAG can identify edges that are different between the two DAGs $\cD\up 1$ and $\cD\up 2$ and hence provides more information than the union graph, computing the difference-DAG requires knowledge of the membership of each observation to the two disease subgroups, which is not available for many diseases. The estimated PAG $\widetilde\cP_\cup$ based on the combined samples from the two patient groups is shown in Figure~\ref{fig:apoptosis_fci}. It was estimated using FCI with stability selection. 
FCI identified BIRC3 as the node with the highest number of incident bidirected edges; BIRC3 is known to be one of the major disregulated genes in ovarian cancer and an inhibitor of apoptosis~\cite{johnstone2008trail,jonsson2014distinct}.



\textbf{T cell activation.} We also applied our framework to single-cell gene expression data of naive and activated T cells (i.e.~$K=2$, with 298 and 377 samples, respectively) from~\citet{singer2016distinct}. 
Following the analysis in~\citet{wang2018direct}, we performed the analysis on 60 genes that had a fold expression change above 10. The FCI output on these 60 nodes is shown in Section~\ref{section:tcells_fci} in the Appendix.  The following nodes have the highest number of incident bidirected edges, indicating that they may play important roles in T cell activation: CDC6, CDC20, SHCBP1, NKG2A, GZMB4 and KIF2C. All these genes have been discribed before as critical:
CDC6 and CDC20 are essential regulators of the cell division cycle. Shorter cell cycle time for increased proliferation is a hallmark of T cell activation ~\cite{qiao2016mechanism, borlado2008cdc6}. SHCBP1 has been shown to be tightly linked to cell proliferation and strongly correlates with proliferative stages of T cell development~\cite{schmandt1999cloning,buckley2014unexpected}. NKG2A functions to limit excessive activation, prevent apoptosis, and preserve the specific T cell response~\cite{rapaport2015inhibitory}.
GZMB4 has been shown to regulate antiviral T cell response~\cite{salti2011granzyme}.
Finally, the gene KIF2C encodes a Kinesin-like protein that functions as a microtubule-dependent molecular motor. It is over-expressed in a variety of solid tumors and induces frequent T cell responses~\cite{gnjatic2010ny}.

\section{DISCUSSION}
In this paper, we provided a graphical representation (via the mixture DAG) of distributions that arise as  mixtures of causal DAGs. We showed that the mixture DAG not only satisfies the Markov property with respect to such mixture distributions, but is also always realizable by a mixture distribution, meaning that it cannot be made sparser without losing the Markov property. In addition, we characterized  the output of the prominent FCI algorithm
when applied to data from such mixture distributions. FCI is a natural candidate in this setting due to the presence of the latent mixing variable. We proved that FCI can identify variables whose conditionals vary across the different components and showed how this property can be used to infer cluster membership of samples. This is relevant for many applications, as for example when studying  diseases consisting of multiple not well characterized subtypes. In such studies, genomic perturbation experiments can now be performed rela- tively routinely, leading to high-throughput interventional data. In future work it would be interesting to study how interventional data could be used to enhance causal inference based on mixtures of DAGs or which interventions to perform in order to enhance identifiability of pathways that are shared among the different subtypes as well as those that are different across the subtypes for personalized interventions.


\subsubsection*{Acknowledgements}
Basil Saeed was partially supported by the Abdul Latif Jameel Clinic for Machine Learning in Health at MIT. Caroline Uhler was partially supported by NSF (DMS-1651995), ONR (N00014-17-1-2147 and N00014-18-1-2765), IBM, and a Simons Investigator Award.

\bibliography{references}

\begin{thebibliography}{}

\bibitem[Borlado and M{\'e}ndez, 2008]{borlado2008cdc6}
Borlado, L.~R. and M{\'e}ndez, J. (2008).
\newblock {CDC6: from DNA replication to cell cycle checkpoints and
  oncogenesis}.
\newblock {\em Carcinogenesis}, 29(2):237--243.

\bibitem[Buckley et~al., 2014]{buckley2014unexpected}
Buckley, M.~W., Arandjelovic, S., Trampont, P.~C., Kim, T.~S., Braciale, T.~J.,
  and Ravichandran, K.~S. (2014).
\newblock {Unexpected phenotype of mice lacking SHCBP1, a protein induced
  during T cell proliferation}.
\newblock {\em PloS ONE}, 9(8).

\bibitem[Chickering, 2002]{chickering2002optimal}
Chickering, D.~M. (2002).
\newblock Optimal structure identification with greedy search.
\newblock {\em Journal of Machine Learning Research}, 3(Nov):507--554.

\bibitem[Chu et~al., 2003]{chu2003statistical}
Chu, T., Glymour, C., Scheines, R., and Spirtes, P. (2003).
\newblock A statistical problem for inference to regulatory structure from
  associations of gene expression measurements with microarrays.
\newblock {\em Bioinformatics}, 19(9):1147--1152.

\bibitem[Friedman et~al., 2000]{friedman2000using}
Friedman, N., Linial, M., Nachman, I., and Pe'er, D. (2000).
\newblock Using {B}ayesian networks to analyze expression data.
\newblock {\em Journal of Computational Biology}, 7(3-4):601--620.

\bibitem[Gates and Molenaar, 2012]{gates2012group}
Gates, K.~M. and Molenaar, P.~C. (2012).
\newblock Group search algorithm recovers effective connectivity maps for
  individuals in homogeneous and heterogeneous samples.
\newblock {\em NeuroImage}, 63(1):310--319.

\bibitem[Gnjatic et~al., 2010]{gnjatic2010ny}
Gnjatic, S., Cao, Y., Reichelt, U., Yekebas, E.~F., N{\"o}lker, C., Marx,
  A.~H., Erbersdobler, A., Nishikawa, H., Hildebrandt, Y., Bartels, K., et~al.
  (2010).
\newblock {NY-CO-58/KIF2C} is overexpressed in a variety of solid tumors and
  induces frequent {T} cell responses in patients with colorectal cancer.
\newblock {\em International Journal of Cancer}, 127(2):381--393.

\bibitem[Heckerman et~al., 1995]{heckerman1995real}
Heckerman, D., Mamdani, A., and Wellman, M.~P. (1995).
\newblock Real-world applications of {B}ayesian networks.
\newblock {\em Communications of the ACM}, 38(3):24--26.

\bibitem[Johnstone et~al., 2008]{johnstone2008trail}
Johnstone, R.~W., Frew, A.~J., and Smyth, M.~J. (2008).
\newblock The {TRAIL} apoptotic pathway in cancer onset, progression and
  therapy.
\newblock {\em Nature Reviews Cancer}, 8(10):782--798.

\bibitem[J{\"o}nsson et~al., 2014]{jonsson2014distinct}
J{\"o}nsson, J.-M., Bartuma, K., Dominguez-Valentin, M., Harbst, K., Ketabi,
  Z., Malander, S., J{\"o}nsson, M., Carneiro, A., M{\aa}sb{\"a}ck, A.,
  J{\"o}nsson, G., et~al. (2014).
\newblock Distinct gene expression profiles in ovarian cancer linked to {Lynch}
  syndrome.
\newblock {\em Familial Cancer}, 13(4):537--545.

\bibitem[Lauritzen, 1996]{lauritzen1996graphical}
Lauritzen, S.~L. (1996).
\newblock {\em Graphical Models}, volume~17.
\newblock Clarendon Press.

\bibitem[Pearl, 2009]{pearl2009causality}
Pearl, J. (2009).
\newblock {\em Causality}.
\newblock Cambridge University Press.

\bibitem[Qiao et~al., 2016]{qiao2016mechanism}
Qiao, R., Weissmann, F., Yamaguchi, M., Brown, N.~G., VanderLinden, R., Imre,
  R., Jarvis, M.~A., Brunner, M.~R., Davidson, I.~F., Litos, G., et~al. (2016).
\newblock Mechanism of {APC/CCDC20} activation by mitotic phosphorylation.
\newblock {\em Proceedings of the National Academy of Sciences},
  113(19):E2570--E2578.

\bibitem[Ramsey et~al., 2011]{ramsey2011meta}
Ramsey, J., Spirtes, P., and Glymour, C. (2011).
\newblock On meta-analyses of imaging data and the mixture of records.
\newblock {\em NeuroImage}, 57(2):323--330.

\bibitem[Rapaport et~al., 2015]{rapaport2015inhibitory}
Rapaport, A.~S., Schriewer, J., Gilfillan, S., Hembrador, E., Crump, R.,
  Plougastel, B.~F., Wang, Y., Le~Friec, G., Gao, J., Cella, M., et~al. (2015).
\newblock The inhibitory receptor {NKG2A} sustains virus-specific {CD8+ T}
  cells in response to a lethal poxvirus infection.
\newblock {\em Immunity}, 43(6):1112--1124.

\bibitem[Richardson and Spirtes, 2002]{richardson2002ancestral}
Richardson, T. and Spirtes, P. (2002).
\newblock Ancestral graph {M}arkov models.
\newblock {\em The Annals of Statistics}, 30(4):962--1030.

\bibitem[Rosenberg and Hirschberg, 2007]{rosenberg2007v}
Rosenberg, A. and Hirschberg, J. (2007).
\newblock V-measure: A conditional entropy-based external cluster evaluation
  measure.
\newblock In {\em Proceedings of the 2007 Joint Conference on Empirical Methods
  in Natural Language Processing and Computational Natural Language Learning
  (EMNLP-CoNLL)}, pages 410--420.

\bibitem[Sadeghi et~al., 2013]{sadeghi2013stable}
Sadeghi, K. et~al. (2013).
\newblock Stable mixed graphs.
\newblock {\em Bernoulli}, 19(5B):2330--2358.

\bibitem[Sadeghi and Lauritzen, 2014]{sadeghi2014markov}
Sadeghi, K. and Lauritzen, S. (2014).
\newblock Markov properties for mixed graphs.
\newblock {\em Bernoulli}, 20(2):676--696.

\bibitem[Salti et~al., 2011]{salti2011granzyme}
Salti, S.~M., Hammelev, E.~M., Grewal, J.~L., Reddy, S.~T., Zemple, S.~J.,
  Grossman, W.~J., Grayson, M.~H., and Verbsky, J.~W. (2011).
\newblock Granzyme {B} regulates antiviral {CD8+ T} cell responses.
\newblock {\em The Journal of Immunology}, 187(12):6301--6309.

\bibitem[Schmandt et~al., 1999]{schmandt1999cloning}
Schmandt, R., Liu, S.~K., and McGlade, C.~J. (1999).
\newblock Cloning and characterization of {mPAL, a novel Shc SH2}
  domain-binding protein expressed in proliferating cells.
\newblock {\em Oncogene}, 18(10):1867--1879.

\bibitem[Singer et~al., 2016]{singer2016distinct}
Singer, M., Wang, C., Cong, L., Marjanovic, N.~D., Kowalczyk, M.~S., Zhang, H.,
  Nyman, J., Sakuishi, K., Kurtulus, S., Gennert, D., et~al. (2016).
\newblock A distinct gene module for dysfunction uncoupled from activation in
  tumor-infiltrating {T} cells.
\newblock {\em Cell}, 166(6):1500--1511.

\bibitem[Solus et~al., 2017]{solus2017consistency}
Solus, L., Wang, Y., Matejovicova, L., and Uhler, C. (2017).
\newblock Consistency guarantees for permutation-based causal inference
  algorithms.
\newblock {\em arXiv preprint arXiv:1702.03530}.

\bibitem[Spirtes, 1994]{spirtes1994conditional}
Spirtes, P. (1994).
\newblock Conditional independence properties in directed cyclic graphical
  models for feedback.
\newblock Technical report, Carnegie Mellon University.

\bibitem[Spirtes et~al., 2000]{spirtes2000causation}
Spirtes, P., Glymour, C.~N., and Scheines, R. (2000).
\newblock {\em Causation, Prediction, and Search}.
\newblock MIT press.

\bibitem[Strobl, 2019a]{strobl2019global}
Strobl, E.~V. (2019a).
\newblock The global {M}arkov property for a mixture of {DAGs}.
\newblock {\em arXiv preprint arXiv:1909.05418}.

\bibitem[Strobl, 2019b]{pmlr-v104-strobl19a}
Strobl, E.~V. (2019b).
\newblock Improved causal discovery from longitudinal data using a mixture of
  {DAGs}.
\newblock In {\em Proceedings of Machine Learning Research}, volume 104, pages
  100--133.

\bibitem[Thiesson et~al., 1997]{thiesson2013learning}
Thiesson, B., Meek, C., Chickering, D.~M., and Heckerman, D. (1997).
\newblock Learning mixtures of {DAG} models.
\newblock In {\em Proceedings of the Fourteenth Conference on Uncertainty in
  Artificial Intelligence}, pages 504--513.

\bibitem[Tothill et~al., 2008]{tothill2008novel}
Tothill, R.~W., Tinker, A.~V., George, J., Brown, R., Fox, S.~B., Lade, S.,
  Johnson, D.~S., Trivett, M.~K., Etemadmoghadam, D., Locandro, B., et~al.
  (2008).
\newblock Novel molecular subtypes of serous and endometrioid ovarian cancer
  linked to clinical outcome.
\newblock {\em Clinical Cancer Research}, 14(16):5198--5208.

\bibitem[Wang et~al., 2020]{wang_JCI}
Wang, Y., Segarra, S., and Uhler, C. (2020).
\newblock {High-dimensional joint estimation of multiple directed Gaussian
  graphical models}.
\newblock {\em Electronic Journal of Statistics}.

\bibitem[Wang et~al., 2018]{wang2018direct}
Wang, Y., Squires, C., Belyaeva, A., and Uhler, C. (2018).
\newblock Direct estimation of differences in causal graphs.
\newblock In {\em Advances in Neural Information Processing Systems}, pages
  3770--3781.

\bibitem[Zhang, 2008]{zhang2008causal}
Zhang, J. (2008).
\newblock Causal reasoning with ancestral graphs.
\newblock {\em Journal of Machine Learning Research}, 9(Jul):1437--1474.

\end{thebibliography}

%
%


\newpage
\onecolumn

\section*{APPENDIX}
\appendix
\newcommand{\snum}{S}
\renewcommand{\theequation}{\snum.\arabic{equation}}
\counterwithin{algorithm}{section}
\counterwithin{figure}{section}

\section{Proof of Proposition~\ref{prop:marginal_mag}}
\label{section:proof_prop_marginal_mag}
We begin by recalling the definition of an inducing path from~\citet{richardson2002ancestral}, specialized to ancestral graphs.

\begin{defn} 
  \label{def:inducing}
  A path $v_1,\dots,v_n$ in an ancestral graph $\cG$ is \emph{inducing} if $v_1$ and $v_n$ are not adjacent in $\cG$ and for all $i\in\{2,\dots,n-1\}$, we have
  $$v_{i-1} \leftrightarrow v_i\leftrightarrow v_{i+1} \qquad \textrm{and} \qquad v_i \in \an_{\cG}(\{v_1, v_n\}).$$
\end{defn}

\noindent \citet{richardson2002ancestral} showed the following condition for an ancestral graph to be maximal.

\begin{lemma}[\cite{richardson2002ancestral}]
An ancestral graph $\cM$ is maximal if and only if $\cG$ does not contain any inducing paths. 
\end{lemma}

This allows us to prove Proposition ~\ref{prop:marginal_mag}.

\begin{proof}[Proof of Proposition ~\ref{prop:marginal_mag}]
We show that the graph resulting from Algorithm~\ref{alg:marginalization} does not contain inducing paths. Let $\cM$ be the output of the algorithm. Suppose we have vertices $v_1,\dots,v_n$ where $v_{i-1} \leftrightarrow v_i\leftrightarrow v_{i+1}$ for all $i\in\{2,\dots,n-1\}$ in $\cM$. Then by step 1 of the algorithm, we must have $v_1,\dots,v_n\in\ch_{\cD_\mu}(y)$, implying that $v_1\leftrightarrow_\cM v_n$, and hence the path is not inducing.
\end{proof}

\section{Counter-example for the Markov property of the mother graph}
\label{section:strobl_example}
In the following, we provide a counter-example for the Markov property of the mother-graph representation introduced by~\citet{pmlr-v104-strobl19a,strobl2019global}. 
We first remark that the Markov property in~\citet{strobl2019global} generalizes that of~\citet{pmlr-v104-strobl19a} in the following sense: if the Markov property of the latter is satisfied, then the former is satisfied. 
Hence, we here provide a counter-example for the former, which can serve as a counter-example for both.

\begin{figure}[ht]
    \centering
    \begin{subfigure}[t]{0.32\textwidth}
        \centering
        \includegraphics[height=30mm]{figures/ex1_D1.png}\
        \caption{$\cD\up1$}
        \label{fig_appendix:ex1_D1}
    \end{subfigure}
    \begin{subfigure}[t]{0.32\textwidth}
        \centering
        \includegraphics[height=30mm]{figures/ex1_D2.png}\
        \caption{$\cD\up2$}
        \label{fig_appendix:ex1_D2}
    \end{subfigure}
    \begin{subfigure}[t]{0.32\textwidth}
        \centering
        \includegraphics[height=30mm]{figures/ex1_mother.png}\
        \caption{$\cD_m$}
        \label{fig_appendix:ex1_mother}
    \end{subfigure}
   \caption{(c) shows the mother graph $\cD_m$ associated with the DAGs $\cD\up1$ and $\cD\up2$ in (a) and (b).}
   \label{fig_appendix:ex1}
\end{figure}

We start by recalling a few definitions from~\citet{pmlr-v104-strobl19a} using notation native to our development. Given a mixture of DAGs with distribution $p_\mu$ where $p\up j$ factorizes according to $\cD\up j$, the mother graph $\cD_m=(V_m, D_m)$ has nodes $V_m := [V] \cup \{y\up1,\dots,y\up K\}$ and directed edges
$$D_m := \bigcup_{1\le j\le K}\{y\up j\rightarrow v\up j : v\in V\setminus V_{\textrm{INV}}\} \cup \{u\up j\rightarrow v\up j: u\rightarrow_{\cD\up j} v\}.$$
An example of the mother graph is shown in Figure~\ref{fig_appendix:ex1}.  A variable $c\up j \in [V]$ in the mother graph is called an \emph{m-collider} if and only if at least one of the following conditions hold:
\begin{itemize}
    \item $a\up j\rightarrow c\up j\leftarrow b\up j$, where $a,b\in V\cup \{y\}$
    \item $a \up j\rightarrow c\up j \leftarrow y\up j$ and $y\up k\rightarrow c\up k\leftarrow b\up k$ where $a,b\in V$.
\end{itemize}
An \emph{m-path} exists between $[A]$ and $[B]$ in the mother graph if and only if there exists a sequence of triples between $[A]$ and $[B]$ such that at least one of the following two conditions is true for each triple in the sequence:
\begin{itemize}
    \item $a\up j\startailstar c\up j\startailstar b\up j$ with $a,b,c\in V\cup \{y\}$
    \item $a\up j\rightarrow c\up j \leftarrow y\up j$ and $y\up k \rightarrow c\up k \leftarrow b\up k$ where $a,b,c\in V$.
\end{itemize}
Finally, $[A]$ and $[B]$ are said to be \emph{m-d-connected given} $[C]$ if and only if there exists an m-path between $[A]$ and $[B]$ such that the following two conditions hold:
\begin{itemize}
    \item $c\up j \in [C]$ for every m-collider on the path, where $c\in V$
    \item $a\up j \not\in [C]$ for every non-m-collider on the path, where $c\in V\cup\{y\}$.
\end{itemize}

Now, the Markov property for the mother graph states that if $[A]$ and $[B]$ are not m-d connected given $[C]$ in the mother graph, then $X_A \ci X_B \mid X_C$ in $p_\mu$~\cite{pmlr-v104-strobl19a,strobl2019global}.

We now provide a counter example for this Markov property. For this, consider the mother graph in Figure~\ref{fig_appendix:ex1_mother} over $V=\{1,2,3,4\}$.
Note that according to the definition of m-d-connection, $[\{1\}]$ and $[\{4\}]$ are not m-d-connected given $[\{2,3\}]$. Hence, the Markov property should imply that $X_1 \ci X_4 | X_2,X_3$ in any mixture distribution whose mother graph is as shown. In the following, construct a mixture distribution where this is not satisfied. 

For simplicity, let $p_J(1) = p_J(2) = \frac12$.
Define $p\up 1(x_V)$ as 
\begin{align*}
 p\up 1(x_1) &= \cN(x_1;0,1),\\ 
 p\up 1(x_2|x_1) &= \cN(x_2;x_1,1),   \\
 p\up 1(x_3) &= \cN(x_3;0,1),      \\
 p\up 1(x_4) &= \cN(x_4;0,1),    \\
\end{align*}
and $p\up 2(x_V)$ as 
\begin{align*}
 p\up 2(x_1) &= \cN(x_1;0,1), \\ 
 p\up 2(x_2) &= \cN(x_2;0,1),  \\
 p\up 2(x_3|x_4) &= \cN(x_3;x_4,1),\\
 p\up 2(x_4) &= \cN(x_4;0,1).
\end{align*}
Clearly, $p\up 1(x_V)$ and $p\up 2(x_V)$ factorize according to $\cD\up1$ of Figure~\ref{fig_appendix:ex1_D1} and $\cD\up2$ of Figure ~\ref{fig_appendix:ex1_D2}, respectively. Now, 
\begin{align*}
    p_\mu(x_1,x_2,x_3,x_4) &= \sum_{j\in\{1,2\}} p_J(j) p\up j(x_1,x_2,x_3,x_4)\\
    &= \frac12 \frac{1}{(2\pi)^2}\Big( e^{-\frac{x_1^2}{2}} e^{-\frac{x_3^2}{2}} e^{-\frac{x_4^2}{2}} e^{-\frac{(x_2-x_1)^2}{2}}  +
    e^{-\frac{x_1^2}{2}} e^{-\frac{x_2^2}{2}} e^{-\frac{x_4^2}{2}} e^{-\frac{(x_3-x_4)^2}{2}} \Big) \\
    &= \frac12 \frac{1}{(2\pi)^2}e^{-\frac{x_1^2}{2}} e^{-\frac{x_2^2}{2}} e^{-\frac{x_3^2}{2}}e^{-\frac{x_4^2}{2}}
    \Big( e^{x_2x_1}e^{-\frac{x_1^2}{2}} + 
    e^{x_3x_4}e^{-\frac{x_4^2}{2}}\Big),
\end{align*}
which cannot be written as
$$f(x_1,x_2,x_3) g(x_2,x_3,x_4)$$
for any $f,g$, implying that $X_1 \nci X_4 \mid X_2,X_3$ in $p_\mu$. \hfill\qed

\section{Proof of Lemma~\ref{lemma:technical}}

\begin{proof}[Proof of Lemma~\ref{lemma:technical}]
By the assumption, $ p\up {j_1}(x_V)$ factors according to $\cD\up {j_1}$. Hence, it is sufficient to define a distribution $\widetilde p_{X_V,J}(x_v,j)$ over $X_V \cup \{J\}$ that factors according to $\widetilde\cD\up j$, with $J\in\{j_1,j_2\}$ for an arbitrarily chosen $j_2\in\{1,\dots,K\} \setminus \{j_1\}$, such that
$$\widetilde p_{X_V|J}(x_V|j_1)= p\up {j_1}(x_V).$$
Then, the factorization with respect to $\widetilde\cD\up {j_1}$ along with the two d-separation statements in the hypothesis of the lemma would imply 
\begin{align*}
 p\up {j_1}(x_A, x_B| x_C) &= \frac{\sum_{x_{V\setminus (A\cup B\cup C)}} p\up {j_1}(x_V )}{\sum_{x_{V\setminus C}} p\up {j_1}(x_V)}\\
 &= \frac{\sum_{x_{V\setminus (A\cup B\cup C)}}\widetilde p(x_V |j_1)}{\sum_{x_{V\setminus C}}\widetilde p(x_V |j_1)}\\
 &= \widetilde p(x_A,x_B|x_C, j_1)\\
 &= \widetilde p(x_A| x_C)\widetilde p(x_B|x_C,j_1).\\
\end{align*}
To complete the proof, we define such a distribution $\widetilde p$. First let $V_y := \ch_{\widetilde \cD\up {j_1}}(y)$ and note that
\begin{align*}
\widetilde p(x_V,j) &= \widetilde p_J(j)\prod_{v\in V}\widetilde p(x_v | x_{\pa_{\widetilde\cD\up {j_1}}(v)}, j)\\
&=\widetilde p_J(j)\prod_{v\in V_y}\widetilde p(x_v | x_{\pa_{\cD\up {j_1}}(v)},j) \prod_{v\in V\setminus V_y}\widetilde p(x_v | x_{\pa_{\cD\up {j_1}}(v)}).
\end{align*}
Define
$$\widetilde p_J(j) := \begin{cases}
                     p_J(j_1) & j=j_1\\
                    1- p_J(j_1) & j=j_2
                   \end{cases}\;,
$$

$$\widetilde p(x_v | x_{\pa_{\cD\up j}(v)}) :=  p(x_v | x_{\pa_{\cD\up j}(v)}) \quad \forall v\in V\setminus V_y.$$

Now, for each $v\in V_y$, define
$$U(v):= \pa_{\cD\up {j_1}}(v)\cap\pa_{\cD\up {j_2}}(v)$$
and 
$$
D(v):= \pa_{\cD\up{j_2}(v)}\setminus\pa_{\cD\up {j_1}(v)},
$$
and choose an arbitrary fixed value for $x_{\pa_{\cD\up i(v)}\setminus\pa_{\cD\up j(v)}}$ and  denote it by $x_d'(v)$. 

\noindent Then define for all $v\in V_y$,
$$\widetilde p(x_v | x_{\pa_{\cD\up j}(v)},j) := \begin{cases}
                                             p_{X_v|X_{\pa_{\cD\up {j_1}}(v)},J}(x_v | x_{\pa_{\cD\up {j_1}}(v)}, j_1) & j=j_1 \\
                                            
                                             p_{X_v|X_{U(v)},X_{D(v)}, J}(x_v | x_{U(v)}, x_d'(v), j_2) & j=j_2
                                            \end{cases}.
$$
Now, one easily checks that this distribution indeed satisfies the factorization property, which completes the proof.
\end{proof}

%
%
%

\section{Proof of Lemma~\ref{lemma:union_maximal_ancestral}}
\label{section:proof_union_maximal_ancestral}
The ancestral property follows directly since we impose the order compatibility assumption of Definition~\ref{def:ordering}. In the following, we show maximality using the definition of inducing path and the associated maximality condition in Section~\ref{section:proof_prop_marginal_mag}. 

\begin{proof}[Proof of Lemma~\ref{lemma:union_maximal_ancestral}]
Suppose we have a path $v_1\leftrightarrow v_2 \leftrightarrow \dots v_{n-1}\leftrightarrow v_n$ in $\cM_\cup$. 
Then, for all $m\in\{1,\dots,n-1\}$, we must have some $j\in\{1,\dots,K\}$ such that $v_m\up j\leftrightarrow v_{m+1}\up j$ in $\cM\up j$, implying that for all $m$, we must have a $j$ 
such that $v_m\up j,v_{m+1}\up j \in \ch_{\cD\up j}(y)$ and hence a $j$ such that $v_m\up j,v_{m+1}\up j \in \ch_{\cD_\mu}(y)$. But by construction of $\cD_\mu$, this implies that $v_m\up j v_{m+1}\up j\in \ch_{\cD_\mu}(y)$ for all $j\in\{1,\dots,K\}$. Therefore, for any $j$, we have $v_1\up j\cdots,v_n\up j \in \ch_{\cD\up j}(y)$, and hence Algorithm~\ref{alg:marginalization} adds an edge between $v_1\up j$ and $v_n\up j$ in $\cM\up j$, resulting in an edge between $v_1$ and $v_n$ in $\cM_\cup$. Therefore, the path $v_1,\dots, v_n$ is not inducing in $\cM_\cup$.
\end{proof}

\section{Proof of Theorem~\ref{thm:mixture-union}}
\label{section:proof_mixture-union}

Since we assume that $A,B,C\subseteq V$, i.e., these sets do not contain $y$, then $[A]$ and $[B]$ are d-separated in $\cD_\mu$ given $[C]$ if and only if they are d-separated in the marginal MAG of $\cD_\mu$ w.r.t.~$\{y\}$ obtained from Algorithm~\ref{alg:marginalization}. We refer to this MAG as the \emph{mixture MAG} and denote it by $\cM_\mu$. We will make use of this MAG in  parts of the following proof since it simplifies the arguments.

\begin{figure}[!t]
    \centering
    \begin{subfigure}[t]{0.24\textwidth}
        \centering
        \includegraphics[height=30mm]{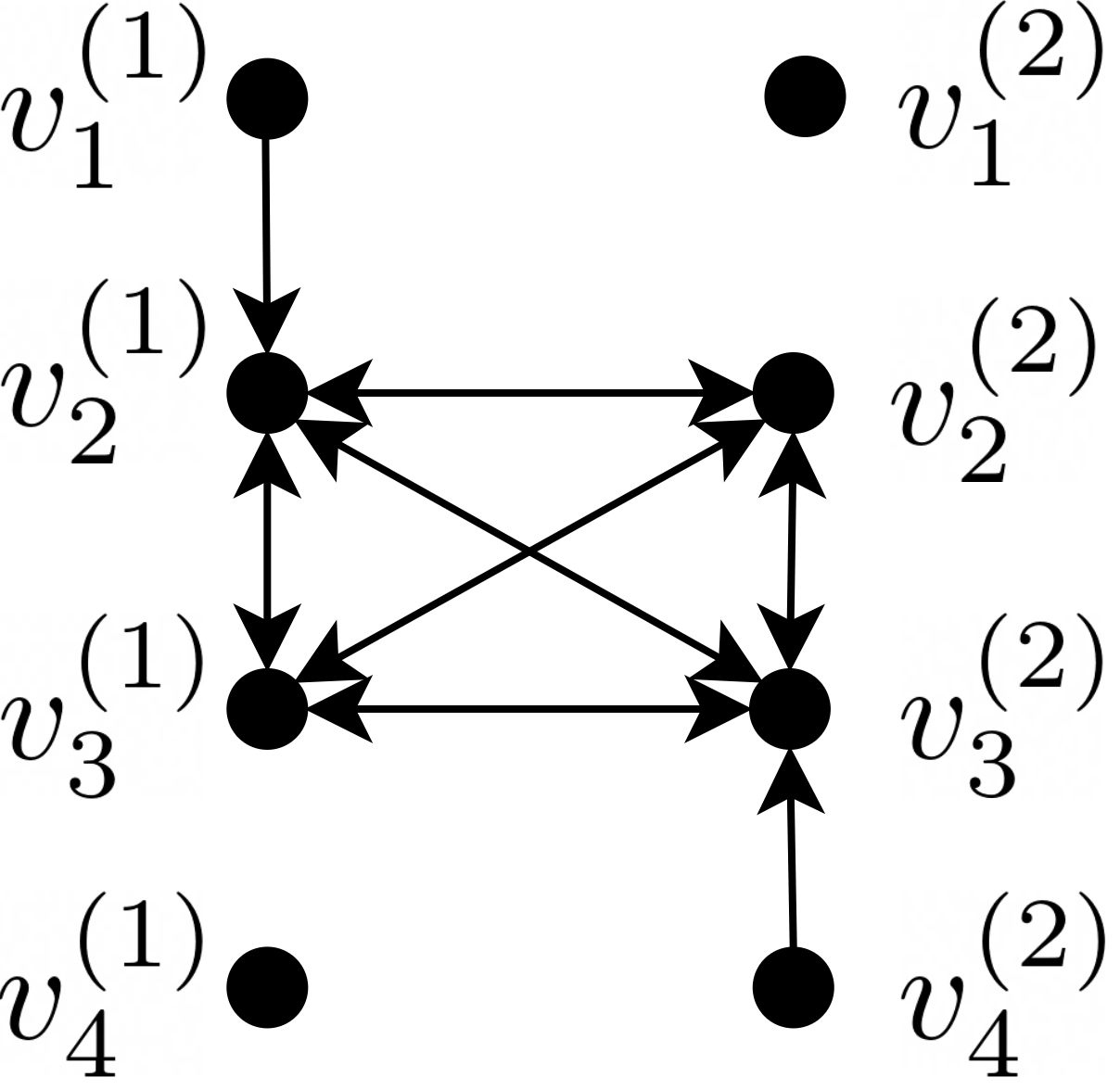}\
        \caption{$\cM_\mu$}
    \end{subfigure}%
   ~ 
    \begin{subfigure}[t]{0.24\textwidth}
        \centering
        \includegraphics[height=30mm]{figures/ex1_M1.png}\
        \caption{$\cM\up 1$}
    \end{subfigure}%
   ~ 
    \begin{subfigure}[t]{0.24\textwidth}
        \centering
        \includegraphics[height=30mm]{figures/ex1_M2.png}\
        \caption{$\cM\up 2$}
    \end{subfigure}%
    ~
    \begin{subfigure}[t]{0.24\textwidth}
        \centering
        \includegraphics[height=30mm]{figures/ex1_union.png}\
        \caption{$\cM_\cup$}
    \end{subfigure}%
   \caption{} 
   \label{fig:ex1_mixture_mag}
\end{figure}

One thing to note about $\cM_\mu$ is that if we remove the edges of the form $u\up j\circtailcirc v\up i$ for $u,v\in V$ and $i\neq j$, then we obtain a bijection between the edges of $\cM_\mu$ and the union of all the edges of $\cM\up j$ for all $j$.
Figure~\ref{fig:ex1_mixture_mag} illustrates this for an example. 
Hence, we can alternatively think of the union graph as having directed edges 
   $$D_\cup := \{u\rightarrow v : u,v \in V, \; \exists_{i}\; u\up i \rightarrow_{\M_\mu} v\up i\},$$
  and bidirected edges
   $$B_\cup := \{u\leftrightarrow v : u,v\in V, \; \exists_{i}\; u\up i \leftrightarrow_{\M_\mu} v\up i\}.$$
   
%
%
%
%
%
%
%
%
%

We prove Theorem~\ref{thm:mixture-union} in 3 main steps.  
\emph{First}, in Lemma~\ref{lemma:first} we show that for any d-connecting path between $a$ and $b$ given $C$ in $\cM_\cup$, we can find a d-connecting path between $a\up i$ and $b\up k$ given $[C]$ in $\cM_\mu$.
\emph{Second}, in Lemma~\ref{lemma:second} we show the converse: that for any d-connecting path $a\up i$ and $b\up k$ given $[C]$ in $\cM_\mu$, we can find a d-connecting path between $a$ and $b$ given $C$ in $\cM_\cup$.
\emph{Finally}, in Lemma~\ref{lemma:third} we show that this equivalence implies that for any disjoint sets $A,B,C\subseteq V$, $A$ and $B$ are d-separated in $\cM_\cup$ if and only if $[A]$ and $[B]$ are d-separated in $\cM_\mu$ given $[C]$.

The proof strategy in Lemmas~\ref{lemma:first} and \ref{lemma:second} relies on concatenating d-connecting paths given $C$ of the form $P_1 = \langle v_1,\dots,v_n\rangle$ and $P_2 = \langle v_n\dots,v_m\rangle$ together to create longer d-connecting paths given $C$ of the form $P =\langle v_1,\dots,v_m\rangle$.
When doing so, we must take care to ensure that $v_n$ is active on the longer path, i.e., we must ensure that $v_n$ is a collider on the path $P$ if and only if $v_n\in C$. 

\subsection{A connecting path in $\cM_\cup$ implies an analogus one in $\cM_\mu$}
We begin by proving some auxiliary results for step 1.
\begin{lemma}[Bidirected Connections]
\label{lemma:bidirected_connections}
If $a\up i \leftrightarrow_{\cM_\mu} b\up k$ for any $i,k\in\{1,\dots,K\}$,
then $a\up i\leftrightarrow_{\cM_\mu} b\up j$ for 
all $j\in\{1,\dots,K\}\setminus\{i\}$.
\end{lemma}
\begin{proof}
$a\up i \leftrightarrow_{\cM_\mu} b\up k$ implies that $a\up i,b\up k\in\ch_{\cD_\mu}(y)$. 
By construction of $\cD_\mu$, this implies $a\up j,b\up j\in\ch_{\cD_\mu}(y)$ for all $j\in\{1,\dots,K\}$, and hence step 1 of Algorithm~\ref{alg:marginalization} will add the bidirected edges $a\up i\leftrightarrow b\up j$ for all $j\in\{1,\dots,K\}$. Step 3 will only remove it if $a\up i$ and $a\up j$ are ancestors of one another in $\cD_\mu$, which could happen only if $j=i$. Hence, $a\up i\leftrightarrow_{\cM_\mu} b\up j$ for all $j\in\{1,\dots,K\}\setminus\{i\}.$
\end{proof}

\begin{lemma}[Bidirected district]
\label{lemma:bidirected_district}
Assume $a\up i\leftrightarrow_{\cM_\mu} b\up j$ and $c\up k \leftrightarrow_{\cM_\mu} d\up l$.
\begin{itemize}
    \item If $i\neq l$, then $a\up i\leftrightarrow_{\cM_\mu} d\up l$. 
    \item If $ i = l$, then
    \begin{itemize}
        \item $a\up i\leftrightarrow_{\cM_\mu} d\up l$ if neither $a\up i,d\up l$ is an ancestor of another in $\cM_\mu$,
        \item $a\up i\rightarrow_{\cM_\mu} d\up l$ if $a\up i\in\an_{\cM_\mu}(d\up l)$; or
        \item $a\up i\leftarrow_{\cM_\mu} d\up l$ if $d\up l \in\an_{\cM_\mu}(a\up i)$.
    \end{itemize}
\end{itemize}
\end{lemma}

\begin{proof}
$a\up i\leftrightarrow_{\cM_\mu} b\up j$ and $c\up k \leftrightarrow_{\cM_\mu} d\up l$ implies that $a\up \iota, b\up\iota, c\up\iota,d\up\iota \in \ch_{\cD_\mu}(y)$ for all $\iota\in\{1,\dots,K\}$. Hence, step 1 of Algorithm~\ref{alg:marginalization} will add $a\up i\leftrightarrow_{\cM_\mu} d\up l$. If $i\neq l$, then $a\up i$ and $d\up l$ cannot be ancestors of one another, implying that step 3 will not remove this bidirected edge. If $i=l$, then the edge will be removed and replaced with the appropriate directed edge if one of $a\up i$ or $d\up l$ is an ancestor of the other. Otherwise, the bidirected edge will remain.
\end{proof}

\begin{lemma}[Arrow tip lemma]
\label{lemma:arrowtip}
Under the ordering assumption in Definition~\ref{def:ordering}, if a directed edge $a\rightarrow_{\cM_\cup} b$ exists in $\cM_\cup$, then we must have $a^j\rightarrow_{\cM_\mu} b^j$ for some $j$ in $\cM_\mu$.
 If a bidirected edge $a\leftrightarrow_{\cM_\cup} b$ exists in $\cM_\cup$, then we must have $a^j\leftrightarrow_{\cM_\mu} b^j$ for some $j$ in $\cM_\mu$. 
\end{lemma}
\begin{proof}
The proof follows directly from the definition of the union graph.
\end{proof}

\begin{lemma}[Changing Arrowtips Lemma]
\label{lemma:changing_arrowtips}
Under the ordering assumption in Definition~\ref{def:ordering}, if $a\up j\starrightarrow_{\cM_\mu} b\up j$ but not  $a\up k \starrightarrow_{\cM_\mu}b\up k$ (same type of edge) for some $j\neq k$, then we must have $b\up j\leftrightarrow b\up k$. 
\end{lemma}

\begin{proof}
The ordering assumption does not allow $a\up j\rightarrow_{\cM_\mu} b\up j$ and $a\up k \leftrightarrow_{\cM_\mu}b\up k$ (and vice versa). Hence, we must only look at the existence of $a\up j\starrightarrow b\up j$ and the in-existence of an edge between $a\up k$ and $b\up k$.

First, we note that if step 1 of Algorithm~\ref{alg:marginalization} defining $\cM_\mu$ adds $b\up j\leftrightarrow b\up k$, then it will remain since
step 2 does not modify edges but only adds them, 
while step 3 will never remove an edge $b\up j\leftrightarrow b\up k$ since neither can be an ancestor or a descendant of the other in $\cD_\mu$.

Now, if $a\up j\rightarrow_{\cD\up j} b\up j$ but not $a\up k\rightarrow_{\cD\up k}b\up k$ for some $k$, then we must have $b\in V\setminus V^I$ and hence $b\up\iota\in\ch_{\cD_\mu}(y)$ for all $\iota\in\{1,\dots,K\}$ by construction of $\cD_\mu$. Therefore, step 1 of Algorithm~\ref{alg:marginalization} will add $b\up j\leftrightarrow b\up k$. 

For the other case we must check that $a\up j\rightarrowcirc b\up j$ was added by the algorithm that created $\cM_\mu$. In all steps, the algorithm will only add such an edge if $b\in V\setminus V_{\textrm{INV}}$ and hence $b\up j \leftrightarrow b\up k$ must have been added in step 1.
\end{proof}

\begin{lemma}[Step 1]
\label{lemma:first}
Under the ordering compatibility assumption in Definition~\ref{def:ordering},
if there is a connecting path between $a$ and $b$ given some $C\subseteq V\setminus\{a,b\}$ in $\cM_\cup$ ending in an arrow head (or tail respectively) incident to $b$, then there is a connecting path between $a\up i$ and $b\up k$ given $[C]$ in $\cM_\mu$ for some $i,k\in\{1,\dots,K\}$ that also ends in an arrow head (or tail respectively) towards $b\up k$.
\end{lemma}

\begin{proof}
We use induction on the number of edges in the connecting path in $\cM_\cup$.
The base case for 1 edge follows directly from Lemma~\ref{lemma:arrowtip}.

Now assume we have a d-connecting path given $C$ consisting of $m+1$ edges in $\cM_\cup$: $P_\cup =\langle a,\dots, d, b\rangle$ ending in an arrow head (or tail respectively).
Consider the sub-path $\langle a,\dots, d\rangle $ with $m$ edges.
By the inductive hypothesis, there is a path $P_\mu = \langle a\up i,\dots, d\up j \rangle$ in $\cM_\mu$ that is d-connecting given $[C]$, for some $i,j$, ending in the same tip.
In the following, we show that we can always find a path of the form $\langle d\up j,\dots, b\up k\rangle$ for some $k$ that can be joined together with $P_\mu$ to create a path $\langle a\up i,b\up k\rangle$ that is d-connecting given $[C]$.
We do this by considering all the different cases for the tips of the edges $c\startailstar_{\cM_\cup} d$ and $d \startailstar_{\cM_\cup} b$. \\

\noindent Before discussing the different cases, note that if the edge $d\up j \startailstar_{\cM_\mu} b\up k$ exists and is of the same type as the edge $d\startailstar_{\cM_\cup} b$, then we can create the desired d-connecting path $\widetilde P_\mu$ from $a\up j$ to $b\up i$ given $[C]$ by concatenating this edge with $P_\mu$, since:
\begin{align*}
 d\textrm{ is active on }Q_\cup &\Rightarrow \big(d \textrm{ is a collider on } Q_\cup \Leftrightarrow d\in C\big)\\
 &\Rightarrow \big(d\up j \textrm{ is a collider on } \widetilde P_\mu \Leftrightarrow d\up j \in [C]\big)\\
 &\Rightarrow d\up j \textrm{ is active on }\widetilde P_\mu,\\
\end{align*}
where the second implication follows because the path $P_\mu$ ends in the correct type of arrow tip by the inductive hypothesis (I.H.). Hence, in what follows, it is sufficient to 
\begin{equation}
\label{ass:not_same_edge}
\textrm{assume either } d\up j \startailstar b\up k \textrm{ is not in  } \cM_\mu \textrm{ or is not the same edge type as } d \startailstar b \textrm{ in } \cM_\cup.
\end{equation}

\begin{enumerate}
 \item[(i)] case $c\starrightarrow d \leftarrowstar b$ in $\cM_\cup$: 
 
 Since $d$ is a collider on the path $Q_\cup$, we must have $d\in C$, and hence $d\up j\in [C]$ for all $j$.
 Hence the path $P_\mu$ is of the form $\langle a\up i,\dots,\gamma\up\iota,d\up j\rangle$, where $\gamma\up {\iota}\starrightarrow d\up j$ for some $\gamma\in V$ and $\iota\in\{1,\dots,K\}$ by the I.H. Furthermore, by Lemma~\ref{lemma:arrowtip}, we must have $d\up {\iota}\leftarrowstar_{\cM_\mu} b\up {k}$ for some $\iota, k\in\{1,\dots,K\}$. 
 Since we assumed in~\eqref{ass:not_same_edge} that this isn't true for $\iota = j$, then by Lemma~\ref{lemma:changing_arrowtips} we must have $d\up j\leftrightarrow d\up {\iota}$, creating the path 
 $\starrightarrow d\up {j} \leftrightarrow d\up {\iota}\leftarrowstar b\up {k}$ that is d-connected given $[C]$ (recall $d\up j,d\up {\iota}\in [C]$). Concatenating $d\up j \leftrightarrow d\up {\iota} \leftarrowcirc b\up k$ to $P_\mu$ gives the desired d-connecting path $\langle a\up i, \dots, d\up j ,d\up \iota,b\up k\rangle$.
\end{enumerate}

\noindent For the remaining cases, we begin by recalling that the edge $d\startailstar_{\cM_\cup} b$ must exist since $d\up k \startailstar_{\cM_\mu} b\up k$ for some $k$ by Lemma~\ref{lemma:arrowtip}. Now, let $\alpha\up j$ be the node on the path $P_\mu$ closest to $a\up i$ such that all nodes between $\alpha\up j$ and $d\up j$ have the same index $j$, i.e., all of these are contained in the same MAG $\cM\up j$. This means that the node preceding $\alpha\up j$ on this path, call it $\gamma\up \kappa$, either has a different index (i.e., a part of a different $\cM\up \kappa$), or $\alpha\up j = a\up i$.

Call $P\up j_\mu=\langle \alpha\up j,\dots,d\up j\rangle$ the subpath of $P_\mu$ from $\alpha \up j$ to $d\up j$. This path is completely contained in $\cM\up j$. If it is possible to find a path $P\up k_\mu = \langle \alpha\up k ,\dots,d\up k\rangle$ in $\cM\up k$ that is analogous to $P\up j_\mu$ (same types of edges), then we can replace the segment $P\up j_\mu$ of $P_\mu$  with $P\up k_\mu$ to obtain a connecting path between $a\up i$ and $d\up k$ given $[C]$. 
Then, concatenating $d\up k \startailstar b\up k$ gives us the desired connecting path from $a \up i$ to $b\up k$ given $[C]$ in $\cM_\mu$.


Hence, in checking the remaining cases, we further 
\begin{equation}
\label{ass:cant_find_path}
\textrm{assume that it is not possible to find a path } P\up k_\mu \textrm{ in } \cM\up k.
\end{equation}

Therefore, walking along the path $P_\mu \up j$ backwards starting at $d\up j$ until $\alpha\up j$, we will eventually find an edge $\beta\up j \startailstar \delta\up j$ such that $\beta \up k \startailstar \delta\up k$ is not an edge.
Take the first such edge.
Now, if this edge was $\beta\up j\leftrightarrow \delta\up j$,
then by Lemma~\ref{lemma:bidirected_connections}, we must have $\beta\up j \leftrightarrow \delta\up k$, 
implying that we can concatenate the subpath of $P_\mu$ of the form $\langle a\up i,\dots,\beta\up j\rangle$ with $\beta\up j\leftrightarrow \delta\up k$ and the subpath of $P_{\mu}\up k$ of the form $\langle \delta\up k,\dots, b\up k\rangle$ to create the desired d-connecting path given $[C]$. Next we look at the situations where we do not have $\beta\up j \leftrightarrow \delta \up j$, 
considering each remaining case on the arrowheads of $c\startailstar d\startailstar b$ in $\cM_\cup$ separately.

\begin{figure}[!b]
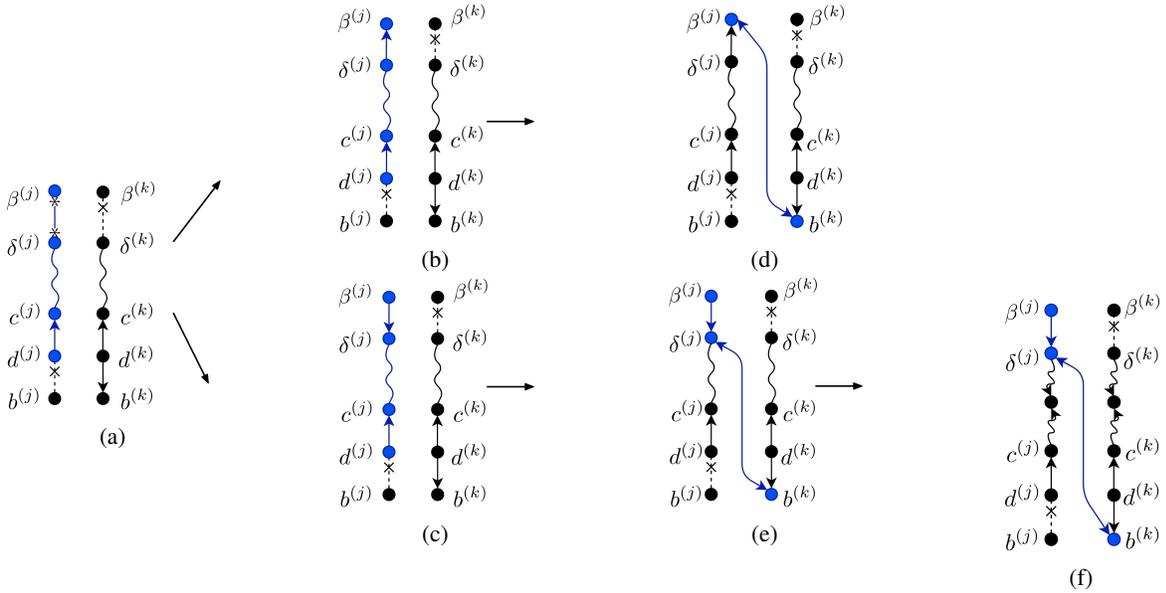

    \centering
    \begin{subfigure}[t]{0.24\textwidth}
        \centering
        \vspace{-7mm}
        \includegraphics[height=1.2in]{figures/case2a.png}\
        \caption{}
        \label{figure:case2a}
    \end{subfigure}%
    ~ 
    \begin{subfigure}[t]{0.24\textwidth}
        \centering
        
    \begin{subfigure}[t]{\textwidth}
        \centering
        \includegraphics[height=1.2in]{figures/case2b.png}
        \caption{}
        \label{figure:case2b}
    \end{subfigure}
    
    \begin{subfigure}[t]{\textwidth}
    \centering
        \includegraphics[height=1.2in]{figures/case2d.png}
        \caption{}
        \label{figure:case2d}
    \end{subfigure}
    \end{subfigure}
   ~ 
    \begin{subfigure}[t]{0.24\textwidth}
        \centering
    \begin{subfigure}[t]{0.9\textwidth}
    \centering
        \includegraphics[height=1.2in]{figures/case2c.png}
        \caption{}
        \label{figure:case2c}
    \end{subfigure}
    \\
    \begin{subfigure}[t]{0.9\textwidth}
    \centering
        \includegraphics[height=1.2in]{figures/case2e.png}
        \caption{}
        \label{figure:case2e}
    \end{subfigure}
    
    \end{subfigure}
   ~ 
    \begin{subfigure}[t]{0.22\textwidth}
        \centering
        \vspace{8mm}
        \includegraphics[height=1.35in]{figures/case2f.png}
        \caption{}
        \label{figure:case2f}
    \end{subfigure}
    \caption{An illustration of the logic in the proof of Lemma~\ref{lemma:first}, case (ii). We do not plot all possible edges in order to reduce clutter. Instead, we plot non-edges using an $\mathsf{x}$ superimposed on a dashed line. Furthermore, we indicate paths between two nodes with a squiggly line. (a), (b) and (c) show the relevant segment of the path $P_\mu$ in blue; (d), (e) and (f) show the segment that replaces $\langle \beta\up j,\dots, d\up j\rangle$ on $P_\mu$ to create the desired d-connecting path in blue.}
    \label{fig:non-ancestral}
\end{figure}\

\begin{enumerate}
    \item[(ii)] case $c\leftarrow d \rightarrow b$ in $\cM_\cup$: This case is depicted in Figure~\ref{figure:case2a}.
    If the first edge found is of the form $\beta\up j \leftarrow \delta\up j$ where $\beta\up k \leftarrow \delta\up k$ is not present (see Figure~\ref{figure:case2b}),
    then by Lemmas~\ref{lemma:changing_arrowtips} and~\ref{lemma:bidirected_district}, we must have $\beta\up j\leftrightarrow b\up k$ (Figure~\ref{figure:case2c}).
    Replacing the segment $\langle \beta\up j,\dots, d\up j\rangle$ of $P_\mu$ with $\beta\up j\leftrightarrow b\up k$ gives the desired path.
   
   Otherwise, if we have $\beta\up j\rightarrow \delta\up j$ instead (Figure~\ref{figure:case2d}), then Lemmas~\ref{lemma:changing_arrowtips} and \ref{lemma:bidirected_district} again say that we must have $\delta\up j\leftrightarrow b\up k$ (Figure~\ref{figure:case2e}).
   The subpath of $P_\mu$ of the form $\langle\delta\up j, c\up j\rangle$ shown in Figure~\ref{figure:case2e} is connecting given $[C]$ by the $I.H.$ Starting at $\delta\up j$ and walking towards $c\up j$, we can find a collider that is in $[C]$ (shown in Figure~\ref{figure:case2f}).
   This collider must be a descendant of $\delta\up j$ Hence, $\delta\up j$ is active given $[C]$ on the path $\beta\up j\rightarrow \delta\up j \leftrightarrow b\up k$ since it is a collider whose descendant is in $[C]$. 
   Replacing the segment $\langle\beta\up j,\dots,d\up j\rangle$ in $P_\mu$ with this path gives the desired connecting path given $[C]$.
\end{enumerate}

\begin{enumerate}
    \item [(iii)] case $c\rightarrow d \rightarrow b$ in $\cM_\cup$:
    Proceeding similarly, if the edge found is of the form $\beta\up j\leftarrow \delta\up j$, then we must have $\beta\up j\leftrightarrow b\up k$ similar to before and for the same reasons.
    Furthermore, we can find a d-connecting path by performing a concatenation similar to the one we did before: replace the segment $\langle\beta\up j,\dots,d\up j\rangle$ of $P_\mu$ with $\beta\up j\leftrightarrow_{\cM_\mu} b\up k$.
    This is illustrated in Figure~\ref{fig:case3a},
   
    If, otherwise, the edge found is of the form $\beta\up j \rightarrow \delta\up j$. We can conclude that we have the bidirected edge $\delta\up j\leftrightarrow d\up k$ by applying the Lemmas~\ref{lemma:bidirected_district} and~\ref{lemma:changing_arrowtips} again. 
    
    If there is a collider on the subpath between $\langle\delta\up j,\dots,c \up j\rangle$, then any such collider must be in $[C]$ since $P_\mu$ is d-connecting given $[C]$
    (see Figure~\ref{fig:case3b}).
    Furthermore, one of these colliders will be a descendant of $\delta\up j$, and we can apply similar logic to that in Case (ii) to show that the path obtained by replacing the segment $\langle\beta\up j,\dots,d\up j\rangle$ of $P_\mu$ with $\delta\up j\leftrightarrow d\up k$ is d-connecting given $[C]$. 
    
     Otherwise, no such collider exists between $\delta\up j$ and $c\up j$ and hence $c\up j$ is a descendant of $\delta\up j$ (see Figure~\ref{fig:case3c}). Therefore, $b\up k$ is a descendant of $\delta\up k$ by the ordering compatibility assumption, and Algorithm~\ref{alg:marginalization}  adds the directed edge $\delta\up k\rightarrow b\up k$ since $\delta\up k$ and $b\up k$ will both be in $\ch_{\cD_\mu}(y)$.
     This further implies that $\delta\up j,b\up j\in\ch_{\cD_\mu}(y)$, so Algorithm~\ref{alg:marginalization} will add an edge between these two nodes. The ordering assumption once again ensures that this edge is of the form $\delta\up j\rightarrow b\up j$.
\end{enumerate}

\begin{figure}[!t]
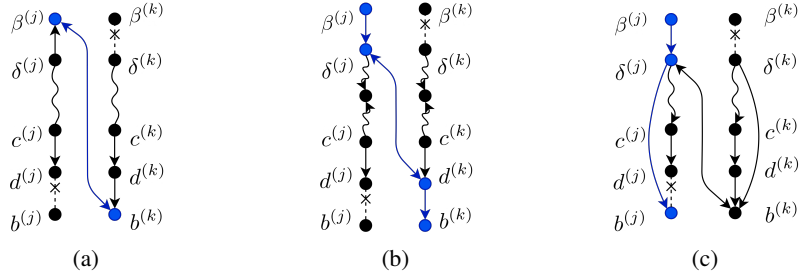

    \centering
    \begin{subfigure}[t]{0.24\textwidth}
        \centering
        \includegraphics[height=1.2in]{figures/case3a.png}\
        \caption{}
        \label{fig:case3a}
    \end{subfigure}%
    \begin{subfigure}[t]{0.24\textwidth}
        \centering
        \includegraphics[height=1.2in]{figures/case3b.png}\
        \caption{}
        \label{fig:case3b}
    \end{subfigure}%
    \begin{subfigure}[t]{0.24\textwidth}
        \centering
        \includegraphics[height=1.2in]{figures/case3c.png}\
        \caption{}
        \label{fig:case3c}
    \end{subfigure}%
    \caption{An illustration of the d-connecting paths constructed by following the logic of case (iii) in the proof of Lemma~\ref{lemma:first}. In each of (a), (b) and (c), the segment that replaces $\langle \beta\up j,\dots, d\up j\rangle$ on $P_\mu$ to create the desired d-connecting path is colored in blue.}
\end{figure}

\begin{enumerate}
    \item [(iv)] case $c\leftarrow d\leftarrow b$ in $\cM_\cup$.
    Proceeding similarly, if we have the edge $\beta\up j \rightarrow \delta \up j$, then we can follow the same logic to create the d-connecting path (see Figure~\ref{fig:case4a}).
    
    Otherwise, $\beta\up j \leftarrow \delta \up j$, and we have the bidirected edge $\beta\up j\leftrightarrow d\up k$, and we again check for colliders between $\beta\up j$ and $d\up j$.
    
    If there is a collider, it will be both in $[C]$ and a descendant of $d\up k$ in$\cM_\mu$, and we can find the desired d-connecting path with the same logic followed previously (see Figure~\ref{fig:case4b}).
    
    If there is no such collider, then $\beta\up j$ will be a descendant of $d\up j$, and using a similar argument to that used for Figure~\ref{fig:case3c}, we can conclude that we have directed edges $\beta\up j\leftarrow_{\cM_\mu} d\up j$ and $\beta\up k\leftarrow_{\cM_\mu} d\up k$ (see Figure~\ref{fig:case4c}).
    In such a scenario, we can repeat the logic for the node $\beta$ in place of the node $c$: we continue walking along the path $P_\mu\up j$ starting from $\beta\up j$ until $\alpha\up j$ is reached or until we find another edge along this path that does not exist on $P_\mu\up k$. If the former happens first, we deal with the case like we would have if $P_\mu\up k$ and $P_\mu\up j$ had identical edges. If the latter happens first, then we recursively repeat the logic of case (iv).
\end{enumerate}

\begin{figure}[!t]
    \centering
    \begin{subfigure}[t]{0.24\textwidth}
        \centering
        \includegraphics[height=1.2in]{figures/case4a.png}\
        \caption{}
        \label{fig:case4a}
    \end{subfigure}%
    \begin{subfigure}[t]{0.24\textwidth}
        \centering
        \includegraphics[height=1.2in]{figures/case4b.png}\
        \caption{}
        \label{fig:case4b}
    \end{subfigure}%
    \begin{subfigure}[t]{0.24\textwidth}
        \centering
        \includegraphics[height=1.2in]{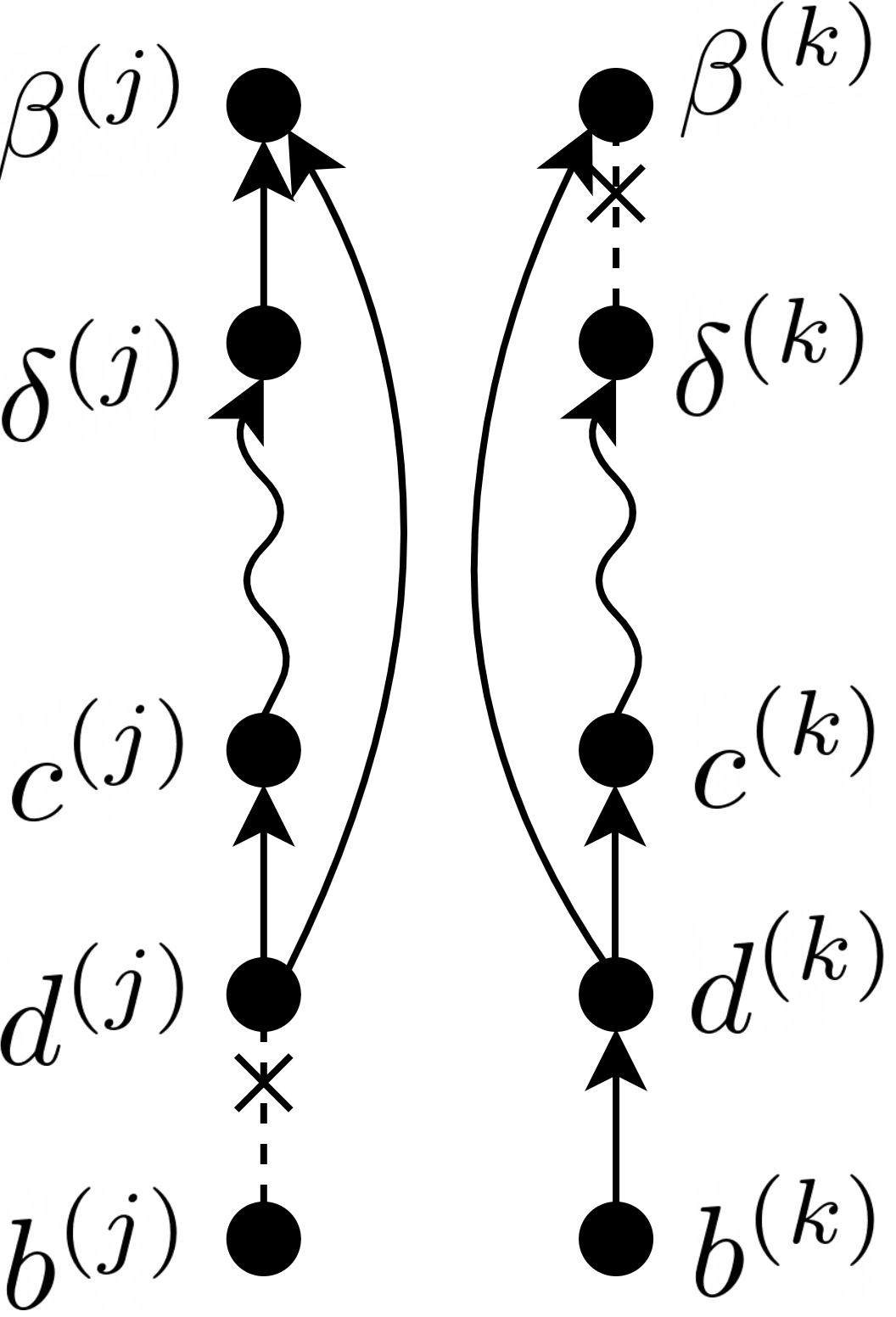}\
        \caption{}
        \label{fig:case4c}
    \end{subfigure}%
    \caption{An illustration of the d-connecting paths constructed by following the logic of case (iv) in the proof of lemma~\ref{lemma:first}. In each of (a) and (b), the segment that replaces $\langle \beta\up j,\dots, d\up j\rangle$ on $P_\mu$ to create the desired d-connecting path is colored in blue.}
\end{figure}
This completes the proof.
\end{proof}

\subsection{A d-connecting path in $\cM_\mu$ implies an analogous d-connecting path in $\cM_\cup$}
Again, we begin with some auxiliary results.

\begin{lemma}[At most 1 bidirected edge]
\label{lemma:one_bidirected_edge}
If there exists a connecting path between $a\up i$ and $b\up k$ given some $[C]$ where $a,b\in V$ and $C\subseteq V\setminus\{a,b\}$ in $\cM_\mu$, then there must exist a path $\widetilde P_\mu$ between $a \up i$ and $b \up k$ that is also connecting given $[C]$ that contains at most one bidirected edge.
\end{lemma}

\begin{proof}
Since $a \up i$ and $b\up k$ are connected given $[C]$ in $\cM_\mu$, then they must also be connected given $[C]$ in $\cD_\mu$. Let $P_\mu$ denote the path connecting $a\up i$ to $b\up k$ given $[C]$ in $\cD_\mu$. Let $P_\mu = \langle a\up i, u_1,\dots, u\up l\rangle$ and let $u_x,u_z$ be the first and last occurrences of the vertex $y$ on $P_\mu$, respectively, if any. Since $y$ has an in-degree of 0, 
neither $u_x$ nor $u_z$ can be a collider. Hence, we can concatentate the paths $P_1 = \langle a\up i,\dots,u_x\rangle$ and $P_2=\langle u_z,\dots,b\up k\rangle$ to get a connecting path given $[C]$ in $\cD_\mu$.

Now, if $u_{x-1}$ is neither an ancestor nor a descendant of $d_{z+1}$, then in $\cM_\mu$, we will have the path
$a\up i, \dots, u_{x-1}, u_{z+1}, \dots, b\up k$ by virtue of Algorithm~\ref{alg:marginalization}, since it adds a bidirected edge between any pair of children of $y$. This is a path from $a\up i$ to $b\up k$ that is also connected given $[C]$ that contains only 1 bidirected edge.

Otherwise, (W.L.O.G) $u_{x-1}\in\an_{\cD_\mu}(u_{z+1})$, i.e., there is a directed path from $u_{x-1}$ to $u_{z+1}$ in $\cD_\mu$. 
Step 3 of Algorithm~\ref{alg:marginalization} adds the edge $u_{x-1} \rightarrow u_{z+1}$ to $\cM_\mu$ to create the path $\widetilde P_\mu := \langle a\up i,\dots,u_{x-1},u_{z+1},\dots,b\up k\rangle$.
This path is from $a\up i$ to $b\up k$ and passes through no bidirected edges. If this path is active, then we are done. If this path is not active, then, 
since $\langle a\up i\dots,u_{x-1}\rangle$ and $\langle u_{z+1},\dots,b\up k$ are active, $\widetilde P_\mu$ must be inactive by virtue of $u_{x-1}\in [C]$. But since $P_\mu$ in $\cD_\mu$ is connecting, this implies that $u_{x-1}$ must have been a collider on that path, hence we have the edge $u_{x-2}\rightarrow u_{x-1}$ in $\cD_\mu$ and $\cM_\mu$. Step 2 of Algorithm~\ref{alg:marginalization} adds $u_{x-2}\rightarrow u_{z+1}$ in such a case. Then, the path $\langle a\up i,\dots,u_{x-1}, u_{z+1},\dots,b\up k\rangle$ must be connecting from $a\up i$ to $b\up k$ given $[C]$, which completes the proof.
\end{proof}

\begin{lemma}[A Connecting Path in $\cM_\mu$ implies a connecting path in $\cM_\cup$]
\label{lemma:second}
Under the assumptionin Defintion~\ref{def:ordering},
if there is a connecting path between $a\up i$ and $b\up k$ given some $[C]$ in $\cM_\mu$  for some $i,k\in\{1,\dots,K\}$, where $C\subseteq V\setminus\{a,b\}$, then there is a connecting path between $a$ and $b$ given $C$ in $\cM_\cup$.
\end{lemma}

\begin{proof}
By Lemma~\ref{lemma:one_bidirected_edge}, we must have a connecting path in $\cM_\mu$ between $a\up i$ and $b\up k$ given $[C]$ that passes through at most 1 bidirected edge. If there exist paths that pass through no bidirected edges, take any such path. Otherwise, take any path that passes through 1 bidirected edge. Call this path $P_\mu= \langle a\up i =: u_0\up i,u_1\up i,\dots,u_{m}\up k := b\up k\rangle$.

By the structure of $\cM_\mu$ discussed in the beginning of this section, only a bidirected edge can connect a node $u_x\up i$ to a node $u_{x+1}\up{k}$ in $\cM_\mu$ for $i\neq  k$. Hence, if there is no bidirected edge on this path, then all the nodes $u_0\up i,\dots,u_m\up i$ will be contained in the same MAG $\cM\up i$. Each edge along this d-connecting path given $[C]$ will show up in $\cM_\cup$, and hence we can create a path $\langle u_0,\dots,u_m \rangle$ that is d-connecting given $C$ in $\cM_\cup$.

In the case where $P_\mu$ contains a bidirected edge, let us label the nodes incident as $u_x\up i\leftrightarrow u_{x+1}\up k$. The segments $\langle u_0\up i,\dots,u_x\up i \rangle$ and $\langle u_{x+1}\up k,\dots,u_{m}\up k\rangle$ will each be contained in $\cM\up i$ and $\cM\up k$ respectively, and hence we can find d-connecting paths $\langle u_0,\dots,u_x\rangle$ and $\langle u_{x+1},\dots,u_{m}\rangle$ in $\cM_\cup$ that are each d-connecting given $C$. We must now show that we can connect these paths to create a d-connecting path given $C$  from $u_0 = a$ to $u_m = b$ in $\cM_\cup$.

Of course, there is no difficulty if the bidirected edge $u_x\leftrightarrow u_{x+1}$ appears in $\cM_\cup$, since we can connect these two subpaths with this bidirected edge and have the desired connecting path. The difficulty is when this edge does not appear. From the definition of $\cM_\cup$, we can see that this only happens when the bidirected edge connects $u_{x}\up i$ and $u_{x+1}\up k$ for $i\neq k$, i.e., the bidirected edge is not contained in any MAG $\cM\up j$ for any $j$. We split the remainder into two cases.

\begin{enumerate}
    \item [(i)] case $u_x = u_{x+1}$.
    If $u_x\up i$ and $u_{x+1}\up k$ are both colliders on $P_\mu$, then we must have $u_x,u_{x+1}\in C$. Then $c=d$ will be an active collider given $C$ in $\cM_\cup$ on the path obtained by concatenating $\langle u_0,\dots,u_x\rangle$ and $\langle u_{x+1},\dots, u_m\rangle$ in $\cM_\cup$, and hence we have our d-connecting path given $C$. We therefore assume, W.L.O.G., that $u_x\up i$ is not a collider on $P_\mu$. 
    
    If there is a path $\langle u_0\up k,\dots, u_x\up k\rangle$ in $\cM_\mu$ where every pair of adjacent vertices $u_n\up k, u_{n+1}\up k$ on this path are connected by the same edge type as the pair $u_n \up i, u_{n+1}\up i$ in $P_\mu$, then we can replace the segment of $\langle u_0\up i,\dots,u_x\up i\rangle$ of $P_\mu$ with $\langle u_0\up k$ to $u_x\up k\rangle$ to obtain a path that is d-connecting given $[C]$ and contained completely in $\cM\up k$, meaning that we can find the desired d-connecting path given $C$ in $\cM_\cup$.
    If no such path exists in $\cM_\mu$, then starting at $u_x\up i$ and walking backwards along $P_\mu$ towards $u_0\up i$, we will find an edge $u_z \up i \startailstar_{\cM_\mu} u_{z+1}\up i$ where $u_z\up k \startailstar_{\cM_\mu} u_{z+1}\up k$ is not an edge. Take the first such edge found (i.e., the edge closest to $u_{x}\up i$ that satisfies this; see Figure~\ref{fig:caseoneaa}). 
    
    If $u_z\up i \rightarrow_{\cM_\mu}u_{z+1} \up i$, then by Lemmas~\ref{lemma:changing_arrowtips} and ~\ref{lemma:bidirected_district}, there is a bidirected edge $u_{z+1}\up i\leftrightarrow_{\cM_\mu} u_x\up k$, implying that  step 1 of Algorithm~\ref{alg:marginalization} adds another bidirected edge $u_{z+1}\up k\leftrightarrow u_x\up k$. 
    If $u_{z+1}\up i$ is not a descendant of $u_x\up i$, then the bidirected edge $u_{z+1}\up k \leftrightarrow_{\cM_\mu} u_x\up k$ would not be removed by step 3 of Algorithm~\ref{alg:marginalization} and hence will appear in $\cM_\mu$. 
    Furthermore, we will have collliders $\alpha\up i$ and $\gamma\up i$ between $u_{z+1}\up i$ and $u_x\up i$ that are in $[C]$ that will be descendants of $u_{z+1}\up i$ and $u_x\up i$ respectively. The ordering assumption ensures that $\alpha$ and $\gamma$ are descendants of $u_{z+1}$ and $u_x$ in $\cM_\cup$, respectively. Hence, the path 
    $\langle u_0,\dots, u_{z+1},u_x,\dots,u_m \rangle$
    in $\cM_\cup$ is d-connecting in $\cM_\cup$ given $C$. Figures~\ref{fig:caseoneaa} and \ref{fig:caseonebb} illustrate this.
    
    \begin{figure}[!t]
    \centering
    \begin{subfigure}[t]{0.24\textwidth}
        \centering
        \includegraphics[height=1.2in]{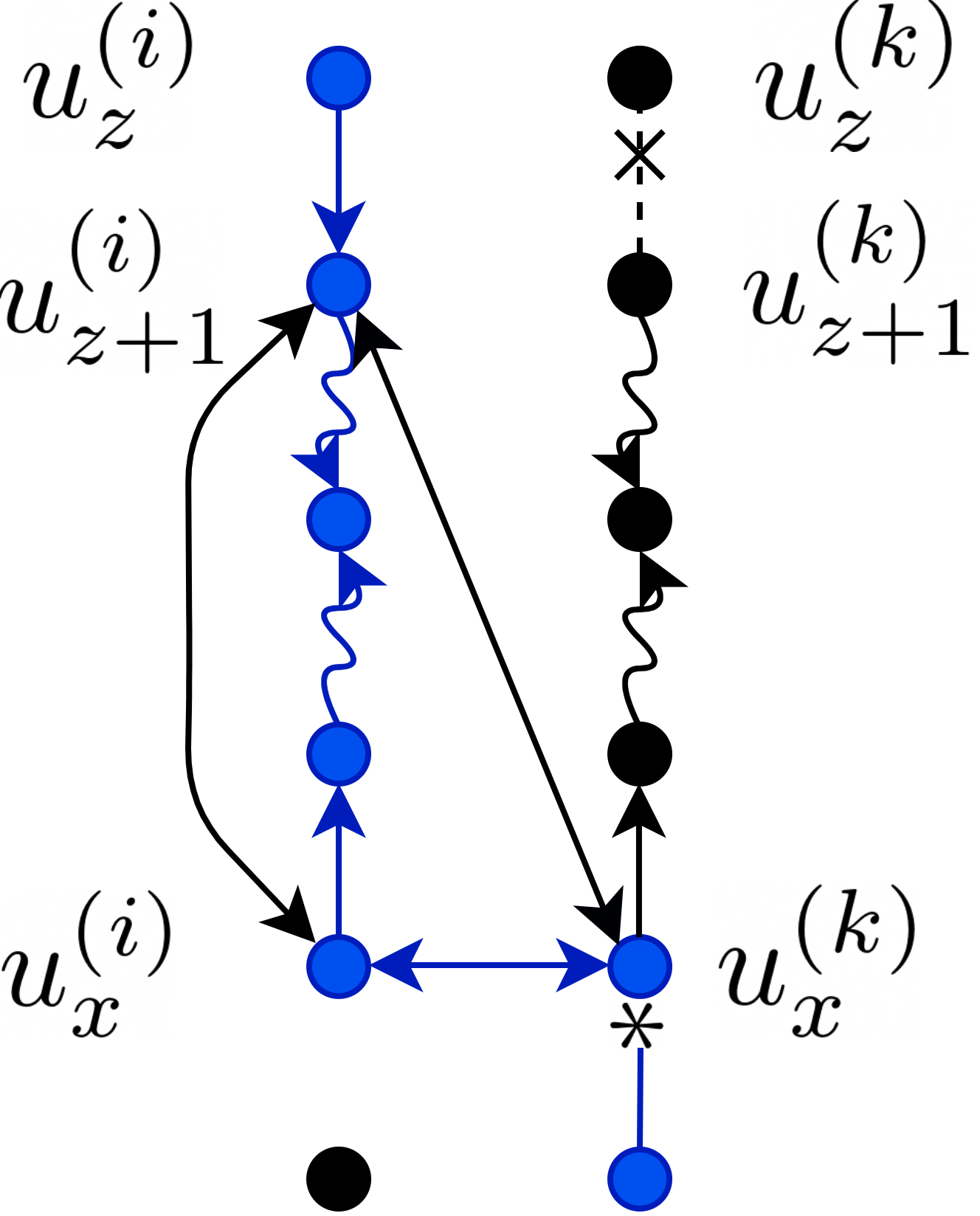}\
        \caption{}
        \label{fig:caseoneaa}
    \end{subfigure}%
    \begin{subfigure}[t]{0.24\textwidth}
        \centering
        \includegraphics[height=1.2in]{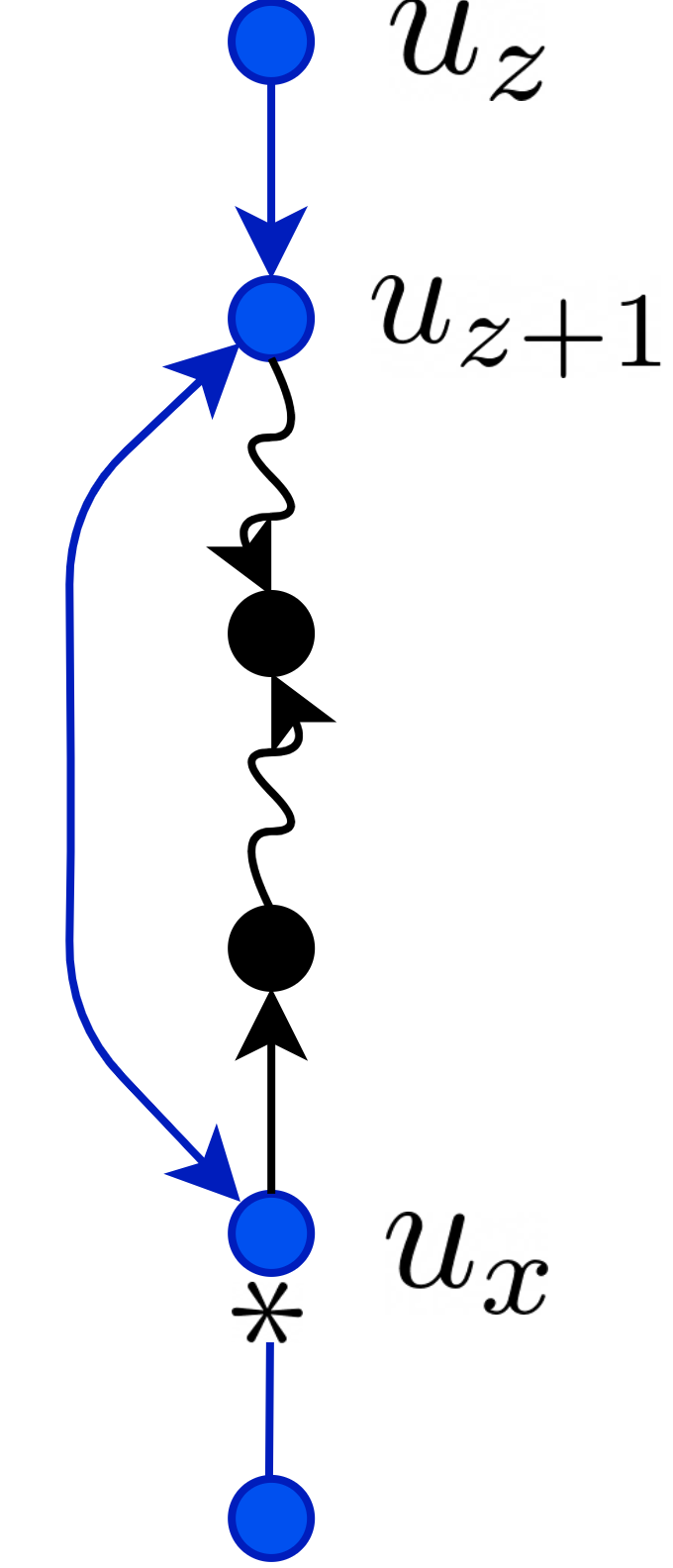}\
        \caption{}
        \label{fig:caseonebb}
    \end{subfigure}%
    \begin{subfigure}[t]{0.24\textwidth}
        \centering
        \includegraphics[height=1.2in]{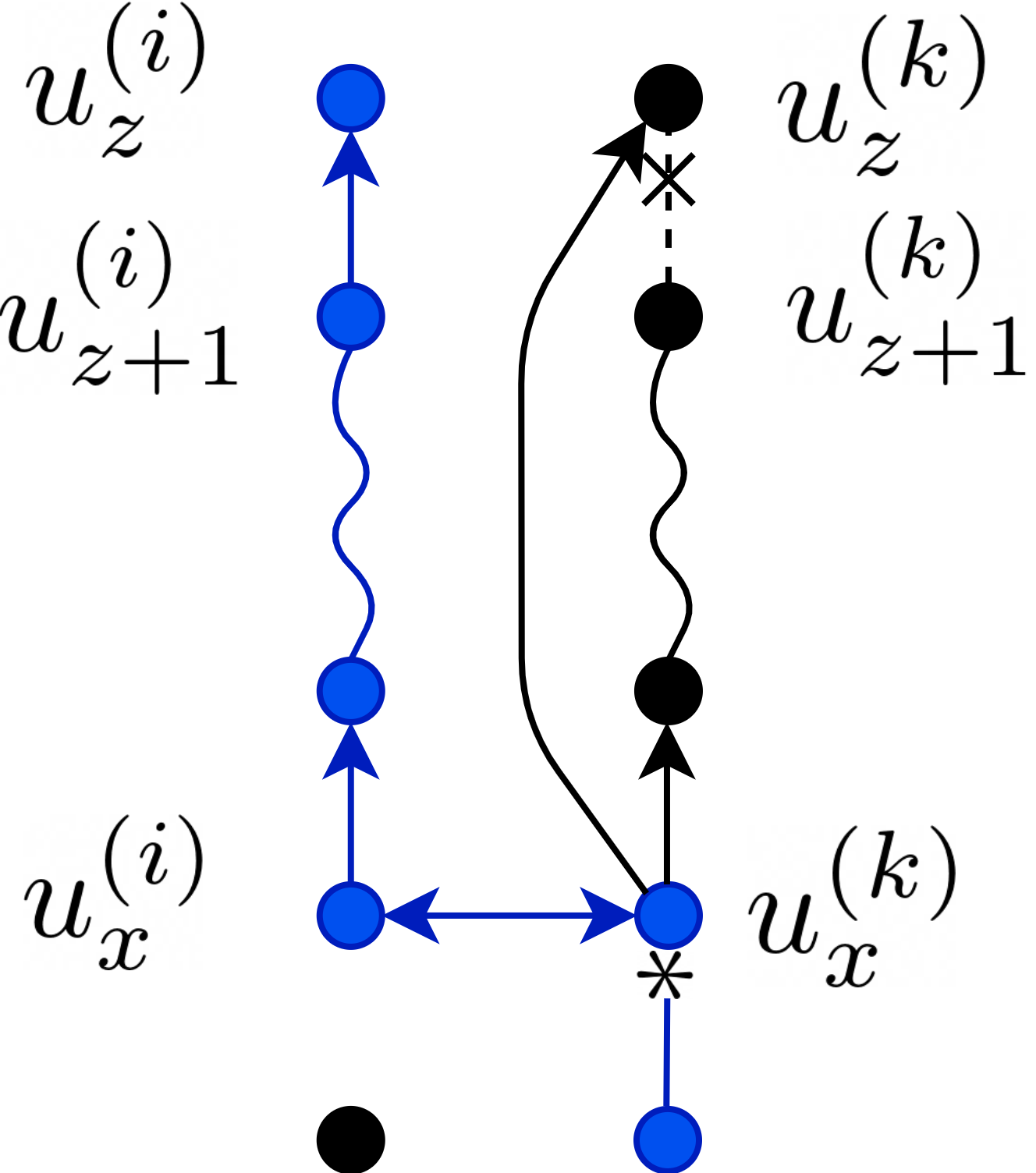}\
        \caption{}
        \label{fig:caseonecc}
    \end{subfigure}%
    \begin{subfigure}[t]{0.24\textwidth}
        \centering
        \includegraphics[height=1.2in]{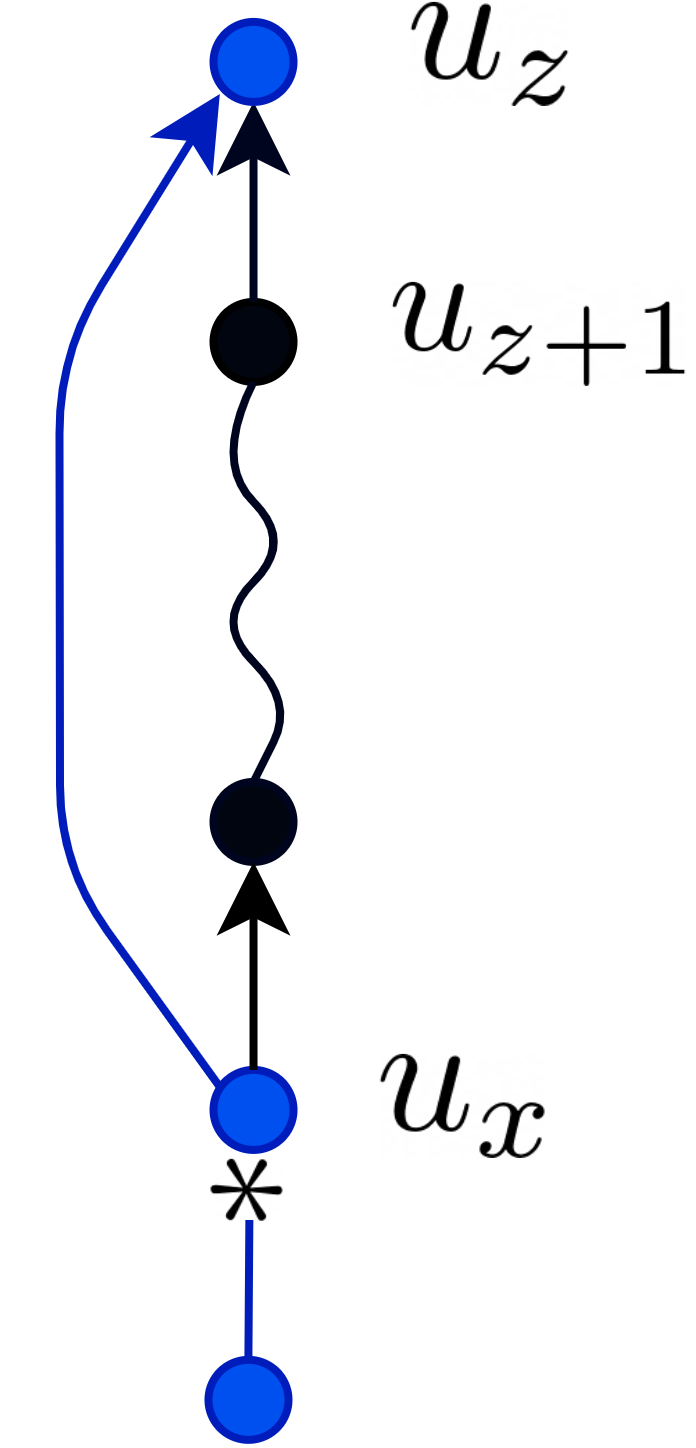}\
        \caption{}
        \label{fig:caseonedd}
    \end{subfigure}%
   \caption{An illustration of the logic for case (i) for the proof of Lemma~\ref{lemma:second}. In (a) and (c), we color in blue the relevant segments of the d-connecting path in $\cM_\mu$, while in (b) and (d), we color in blue the relevant segments of the constructed d-connecting path in $\cM_\cup$.}
\end{figure}
    
    Now we check the case where $u_z\up i \leftarrow u_{z+1}\up i$. 
    If $u_z \up i$ is not a descendant of $u_x\up i$, then we can construct a path in $\cM_\cup$ by a similar argument to the above.
    If $u_z \up i$ is a descendant of $u_x\up i$, then by Lemma~\ref{lemma:bidirected_district}, there is a directed edge $u_z\up i\leftarrow_{\cM_\mu} u_x\up i$, which appears as $u_z\leftarrow_{\cM_\cup} u_x$. We can use this to construct a path in $\cM_\cup$ as shown in Figures~\ref{fig:caseonecc} and \ref{fig:caseonedd}. This path is active since $u_x\up i$ is not a collider, and hence $u_x\up i\not\in [C]$, implying that $u_x\not\in C$.
\end{enumerate}

\begin{enumerate}
    \item[(ii)] case $u_x\neq u_{x+1}$:
    Step 1 of Algorithm~\ref{alg:marginalization} adds the bidirected edge $u_x\up k\leftrightarrow u_{x+1}\up k$, which will show up in $\cM_\cup$ as an edge $u_x \leftrightarrow u_{x+1}$ unless it is removed by step 3; so this is the only case we must check.
    Assume W.L.O.G. that this edge is removed by step 3 because $u_{x}\up k$ is a descendant of $u_{x+1}\up k$ in $\cD_\mu$ and therefore in $\cM_\mu$. 
    Then a directed edge $u_{x}\up i \leftarrow u_{x+1}\up i$ will be added instead, which appears in $\cM_\cup$ as $u_x \leftarrow u_{x+1}$. 
    The only case where we cannot join $\langle u_0,\dots,u_{x}\rangle $ and $\langle u_{x+1},\dots, u_m$ in $\cM_\cup$ together using this directed edge $u_x\leftarrow u_{x+1}$ to create a d-connected path given $C$ is
     when $u_{x+1}\up k$ is in $[C]$, and hence is a collider on $P_\mu$. This implies that we have $u_{x+1}\up k\leftarrow_{\cM_\mu} u_{x+2}\up k$. in which case step 2 of Algorithm~\ref{alg:marginalization} would have added the edge $u_{x+1}\up k\leftarrow u_{x+2}\up k$, which appears as $u_{x+1} \leftarrow_{\cM_\cup} u_{x+2}$.
    This edge can be used to create the d-connecting path given $C$ given by $\langle u_0,\dots, u_{x},u_{x+2},\dots,u_m\rangle$ in $\cM_\cup$. This is illustrated in Figure~\ref{fig:casetwo} and completes the proof.
\end{enumerate}
\vspace{-1.0cm}
\end{proof} 

\begin{figure}[!t]
    \centering
    \begin{subfigure}[t]{0.24\textwidth}
        \centering
        \includegraphics[height=1.2in]{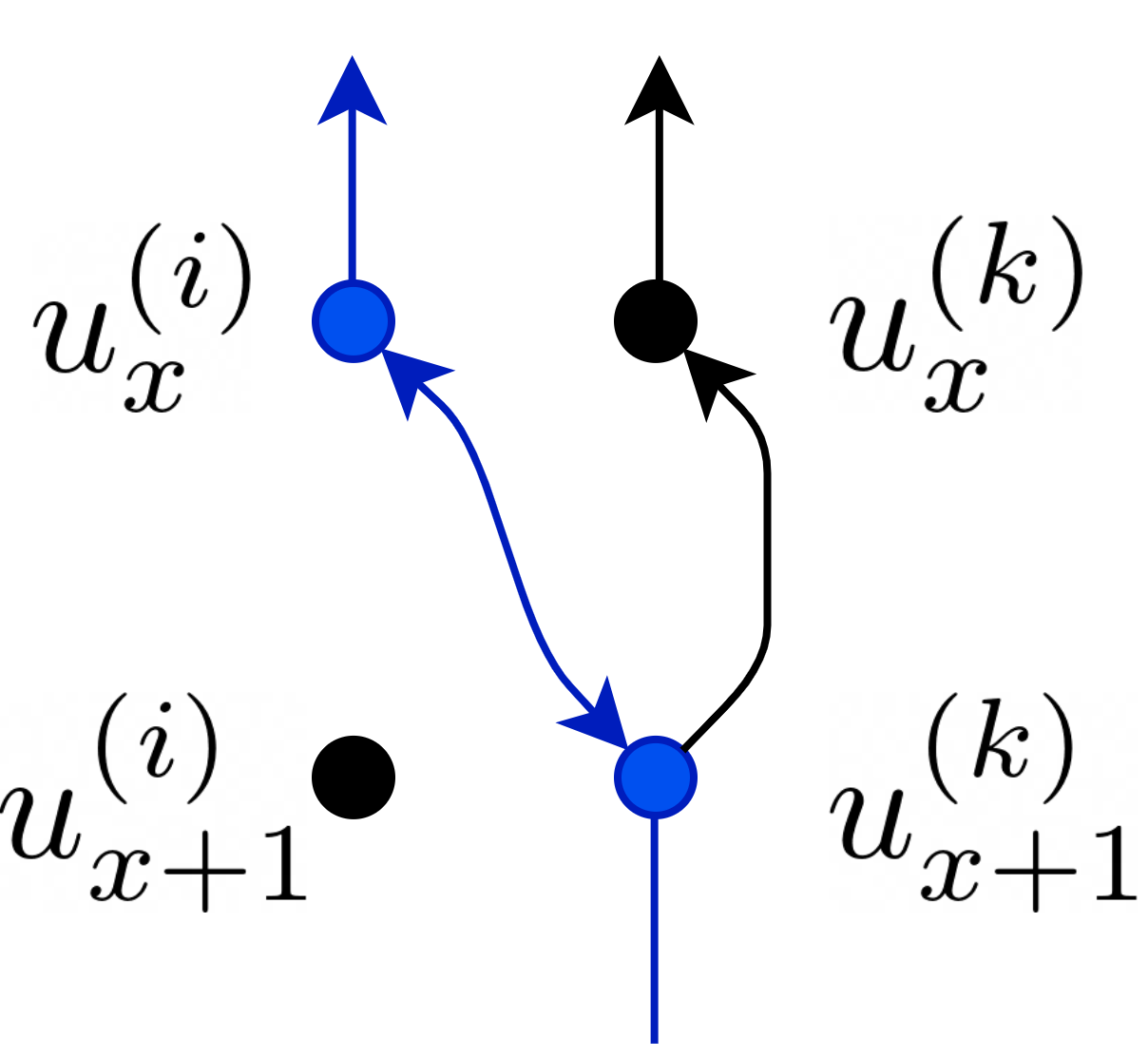}\
        \caption{}
        \label{fig:casetwoa}
    \end{subfigure}%
    \begin{subfigure}[t]{0.24\textwidth}
        \centering
        \includegraphics[height=1.2in]{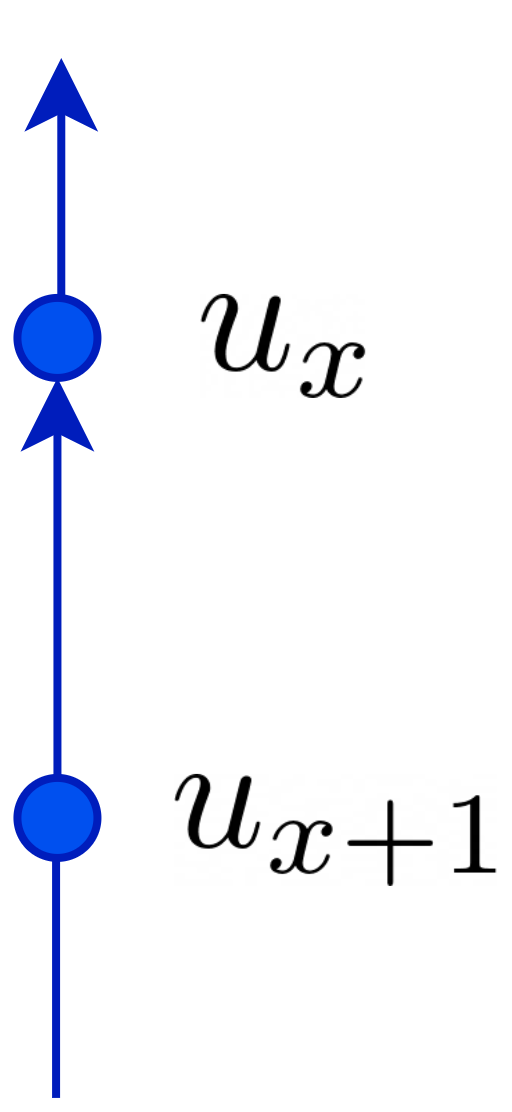}\
        \caption{}
        \label{fig:casetwob}
    \end{subfigure}%
    \begin{subfigure}[t]{0.24\textwidth}
        \centering
        \includegraphics[height=1.2in]{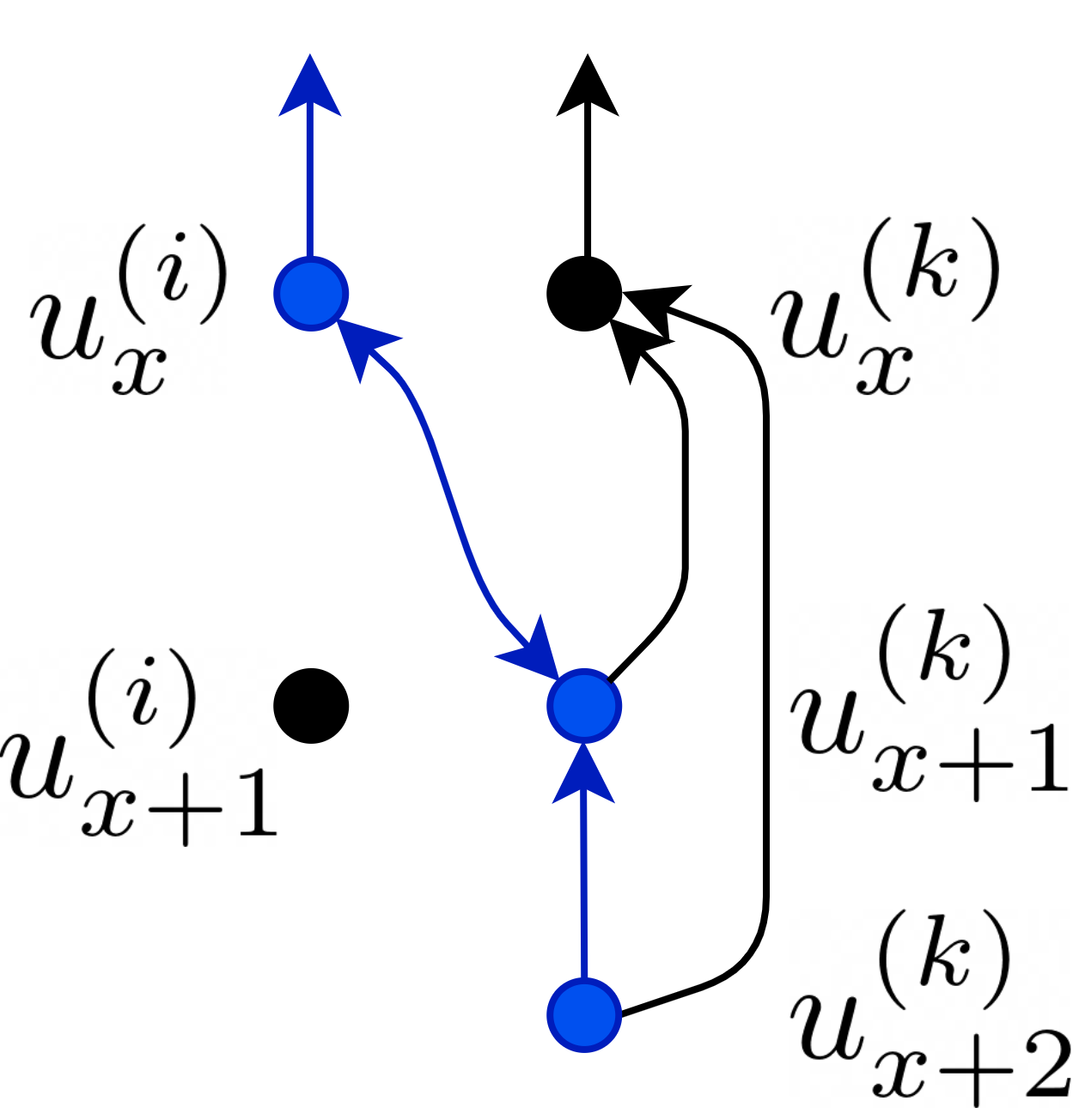}\
        \caption{}
        \label{fig:casetwoc}
    \end{subfigure}%
    \begin{subfigure}[t]{0.24\textwidth}
        \centering
        \includegraphics[height=1.2in]{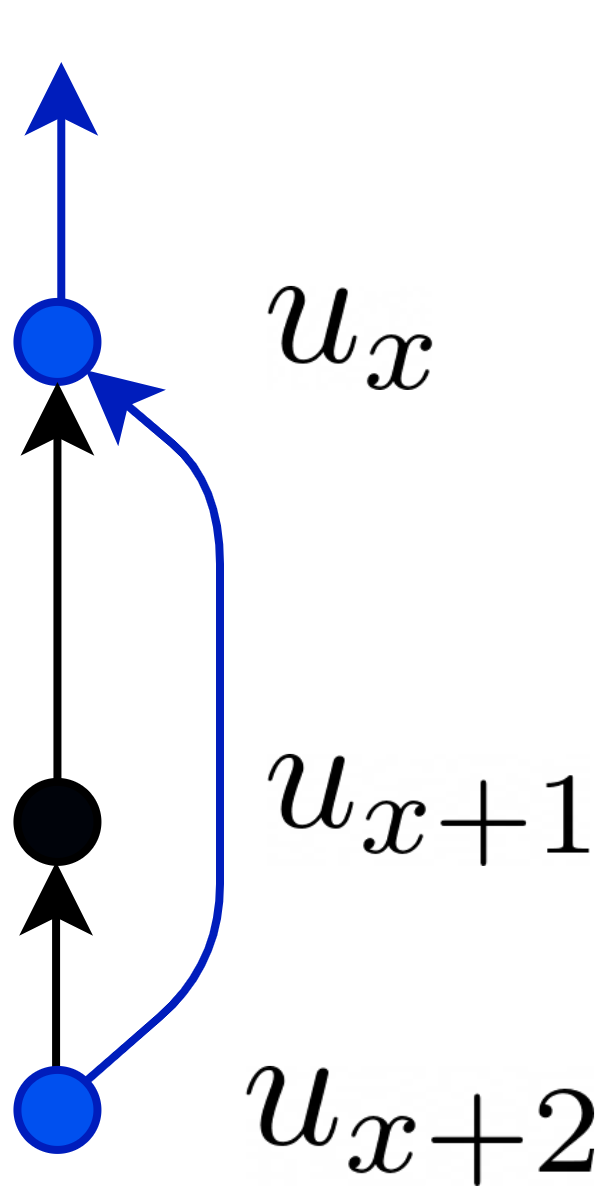}\
        \caption{}
        \label{fig:casetwod}
    \end{subfigure}%
   \caption{An illustration of the logic for case (ii) for the proof of Lemma~\ref{lemma:second}. In (a) and (c), we color in blue the relevant segments of the d-connecting path in $\cM_\mu$, while in (b) and (d), we color in blue the relevant segments of the constructed d-connecting path in $\cM_\cup$.}
           \label{fig:casetwo}
\end{figure}

\subsection{The main result}
Finally, we use the results of the first two steps to prove the following.
\begin{thm}
\label{lemma:third}
Under the assumption in Definition~\ref{def:ordering}, for any disjoint $A,B,C\subseteq V$, 
$[A]$ and $[B]$ are d-separated given $[C]$ in $\cD_\mu$ if and only if $A$ and $B$ are d-separated given $C$ in $\cM_\cup$.
\end{thm}

\begin{proof}
Since $\cM_\mu$ is the marginal MAG in $\cD_\mu$ with respect to the vertex $y$, the d-separation statements involving subsets not including $y$ are the same in both. By proposition~\ref{prop:marginal_mag}, $\cM_\mu$ is a MAG, hence d-separation in $\cM_\mu$ is compositional~\cite{sadeghi2014markov}; therefore for $A,B,C \subseteq V$ disjoint it holds that
 \begin{align*}
 \Big\{[A] \textrm{ sep from } [B] &\textrm{ in } \cM_\mu \textrm{ given } [C] \Big\}\\ & \Leftrightarrow  \Big\{a^i \textrm{ sep from } b^k  \textrm{ in } \cM_\mu \textrm{ given } [C] \textrm{ for all } a^i\in[A],b^k\in[B]\Big\}.
  \end{align*}
 Now Lemmas~\ref{lemma:first} and ~\ref{lemma:second} imply
 \begin{align*}
\Big\{a^i \textrm{ sep from } b^k  \textrm{ in } \cM_\mu \textrm{ given } [C] \textrm{ for all } &a^i\in[A],b^k\in[B]\Big\}\\
&\Leftrightarrow  \Big\{a \textrm{ sep from } b  \textrm{ given } C \textrm{ for all } a\in A,b\in B\Big\}.
 \end{align*}
 Finally, since $\cM_\cup$ is a MAG, applying compositionality gives 
 $$
  \Big\{a \textrm{ sep from } b  \textrm{in } \cM_\cup \textrm{ given } C \textrm{ for all } a\in A,b\in B\Big\} \Leftrightarrow  \Big\{A \textrm{ sep from } B \textrm{ given } C \textrm{ in } \cM_\cup \Big\},
 $$
 which completes the proof.
\end{proof}

\section{Proof of Proposition~\ref{prop:varying}}
\label{section:appendix_varying_proof}
$u\leftrightarrow_{\cM_\cup} v$ implies $u\up j\leftrightarrow_{\cM\up j}v\up j$ for some $j$, which implies $u\leftarrow y\rightarrow v$ in $\cD_\mu$. Hence, $u,v\in V\setminus V_{\textrm{INV}}$. By definition of $V_{\textrm{INV}}$, this implies the claim.

\section{Additional Experimental Results}
\label{appendix:experiments}
\subsection{Synthetic Data}
In the following, we present figures for the experiments described in Section~\ref{section:experiments} for additional values of $K$ and $n$, and when $p(j)$ is not uniform over the mixture components. Figures~\ref{fig:appendix_shd} and ~\ref{fig:appendix_shd_dir} shows the normalized SHD plot in evaluating the union graph as described in the main paper, while figures~\ref{appendix:fig_varying} and ~\ref{appendix:fig_varying_dir} shows the true and false positives in predicted $V\setminus V_{\textrm{INV}}$. Finally, Figure~\ref{appendix:fig_clustering} shows the result of $K$-means clustering.

\begin{figure}[H]
  \centering
  \begin{subfigure}[t]{0.3\textwidth}
    \includegraphics[width=\linewidth]{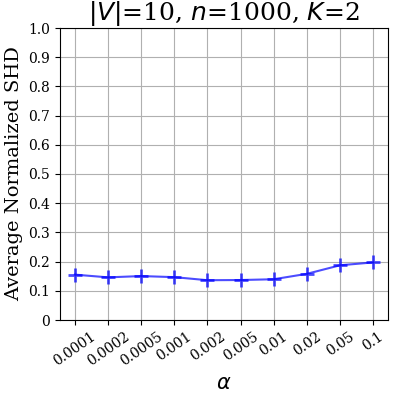}
  \caption{}
  \end{subfigure}
  \begin{subfigure}[t]{0.3\textwidth}
    \includegraphics[width=\linewidth]{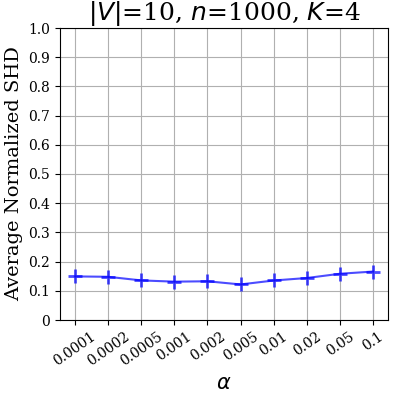}
    \caption{}
  \end{subfigure}
  \begin{subfigure}[t]{0.3\textwidth}
    \includegraphics[width=\linewidth]{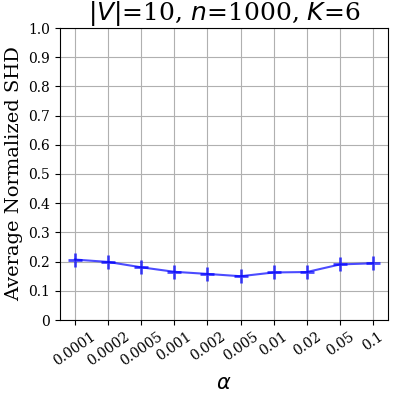}
    \caption{}
  \end{subfigure}
  \\
  \begin{subfigure}[t]{0.3\textwidth}
    \includegraphics[width=\linewidth]{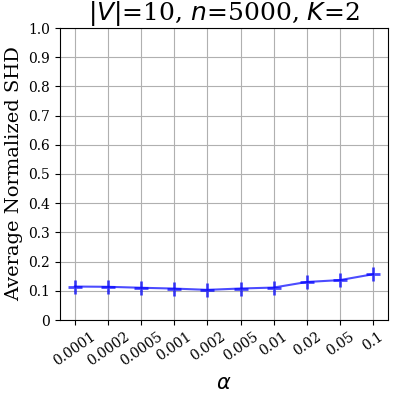}
    \caption{}
  \end{subfigure}
  \begin{subfigure}[t]{0.3\textwidth}
    \includegraphics[width=\linewidth]{figures/SHD-p10-n5000-k4.png}
    \caption{}
  \end{subfigure}
  \begin{subfigure}[t]{0.3\textwidth}
    \includegraphics[width=\linewidth]{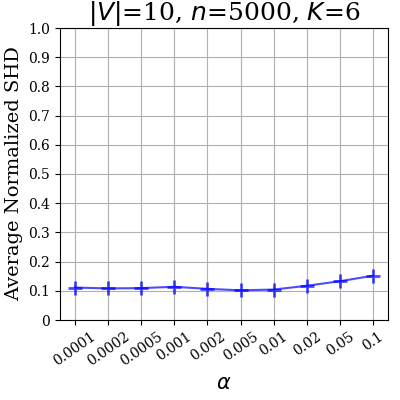}
    \caption{}
  \end{subfigure}
  \\
  \begin{subfigure}[t]{0.3\textwidth}
    \includegraphics[width=\linewidth]{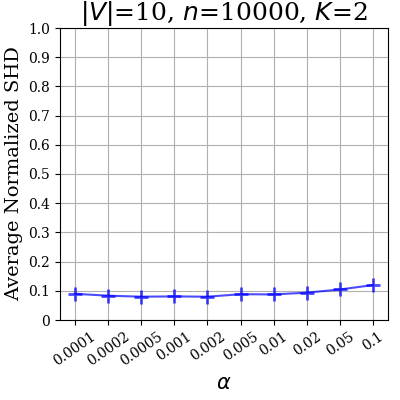}
    \caption{}
  \end{subfigure}
  \begin{subfigure}[t]{0.3\textwidth}
    \includegraphics[width=\linewidth]{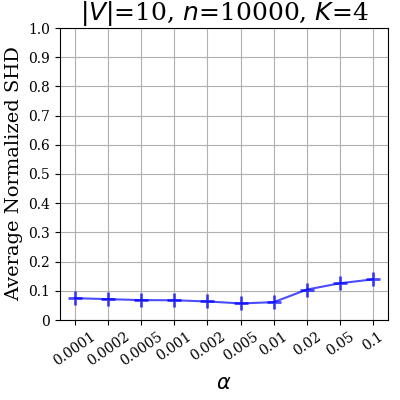}
    \caption{}
  \end{subfigure}
  \begin{subfigure}[t]{0.3\textwidth}
    \includegraphics[width=\linewidth]{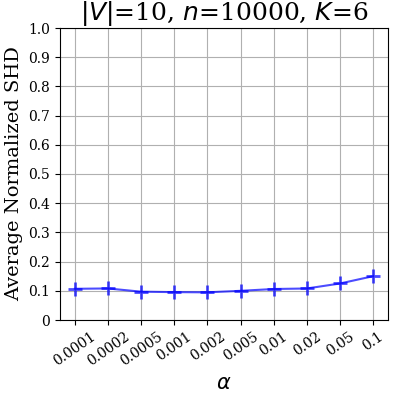}
    \caption{}
  \end{subfigure}
  \caption{Normalized SHD evaluating the estimation of the union graph from mixture data using FCI for $K\in \{2,4,6\}$ and $n\in\{1000,5000,10000\}$. We take $p(j)$ uniform over the mixture components.}
  \label{fig:appendix_shd}
\end{figure}

\begin{figure}[H]
  \centering
  \begin{subfigure}[t]{0.3\textwidth}
    \includegraphics[width=\linewidth]{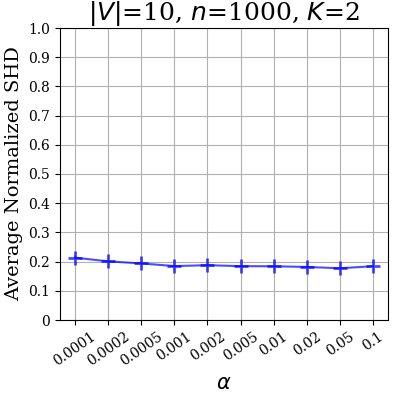}
  \caption{}
  \end{subfigure}
  \begin{subfigure}[t]{0.3\textwidth}
    \includegraphics[width=\linewidth]{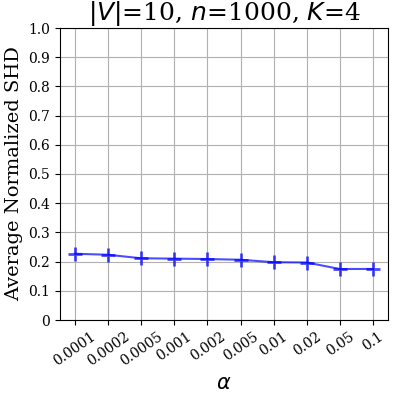}
    \caption{}
  \end{subfigure}
  \begin{subfigure}[t]{0.3\textwidth}
    \includegraphics[width=\linewidth]{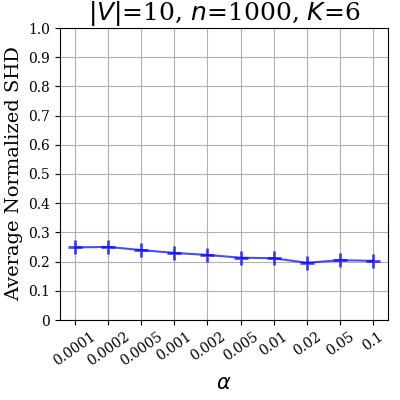}
    \caption{}
  \end{subfigure}
  \\
  \begin{subfigure}[t]{0.3\textwidth}
    \includegraphics[width=\linewidth]{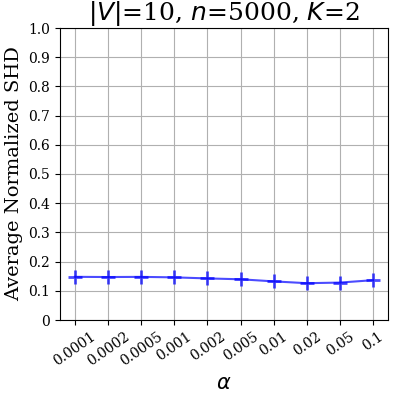}
    \caption{}
  \end{subfigure}
  \begin{subfigure}[t]{0.3\textwidth}
    \includegraphics[width=\linewidth]{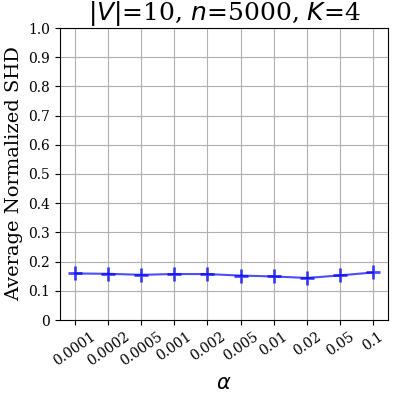}
    \caption{}
  \end{subfigure}
  \begin{subfigure}[t]{0.3\textwidth}
    \includegraphics[width=\linewidth]{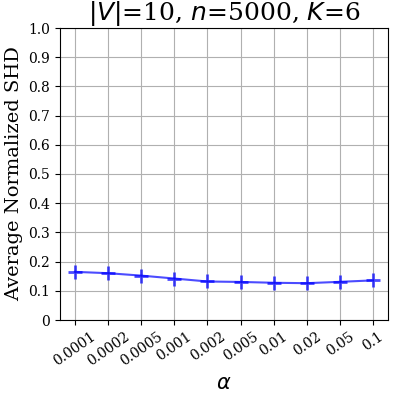}
    \caption{}
  \end{subfigure}
  \\
  \begin{subfigure}[t]{0.3\textwidth}
    \includegraphics[width=\linewidth]{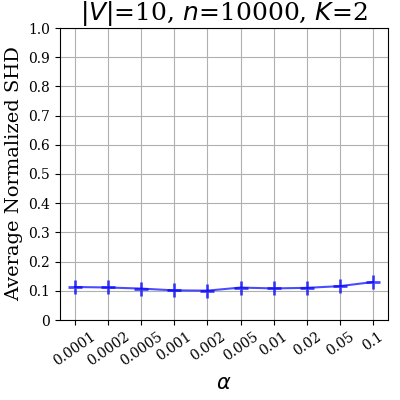}
    \caption{}
  \end{subfigure}
  \begin{subfigure}[t]{0.3\textwidth}
    \includegraphics[width=\linewidth]{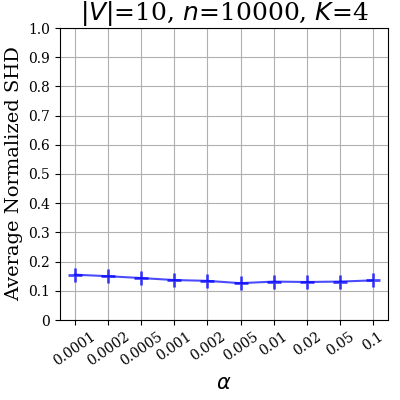}
    \caption{}
  \end{subfigure}
  \begin{subfigure}[t]{0.3\textwidth}
    \includegraphics[width=\linewidth]{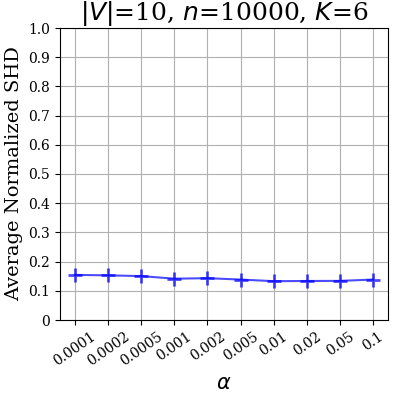}
    \caption{}
  \end{subfigure}
  \caption{Normalized SHD evaluating the estimation of the union graph from mixture data using FCI for $K\in \{2,4,6\}$ and $n\in\{1000,5000,10000\}$. We take $p(j)$ to be Dirichlet with parameter $2$.}
  \label{fig:appendix_shd_dir}
\end{figure}

\begin{figure}[H]
  \centering
  \begin{subfigure}[t]{0.3\textwidth}
    \includegraphics[width=\linewidth]{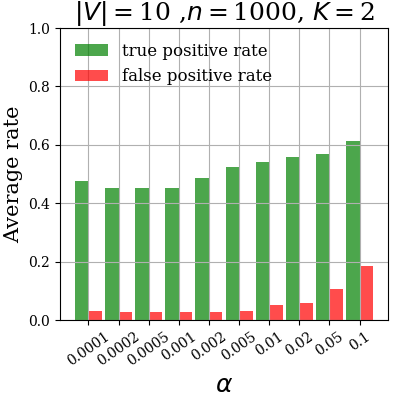}
  \caption{}
  \end{subfigure}
  \begin{subfigure}[t]{0.3\textwidth}
    \includegraphics[width=\linewidth]{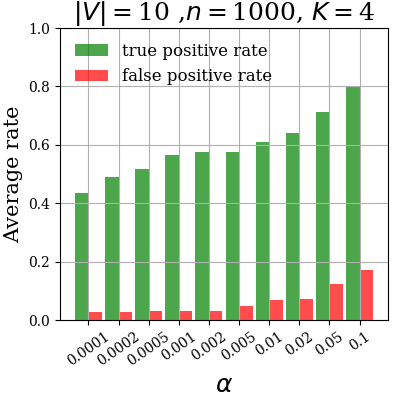}
    \caption{}
  \end{subfigure}
  \begin{subfigure}[t]{0.3\textwidth}
    \includegraphics[width=\linewidth]{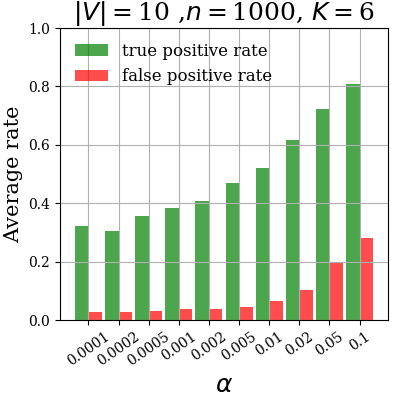}
    \caption{}
  \end{subfigure}
  \\
  \begin{subfigure}[t]{0.3\textwidth}
    \includegraphics[width=\linewidth]{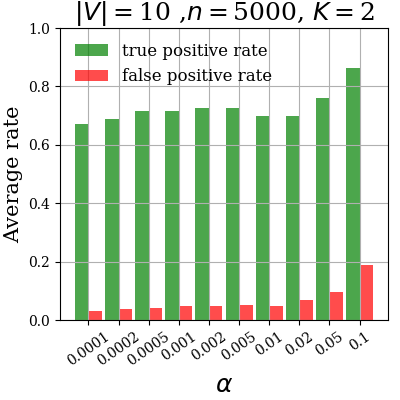}
  \caption{}
  \end{subfigure}
  \begin{subfigure}[t]{0.3\textwidth}
    \includegraphics[width=\linewidth]{figures/EFFECTS-p10-n5000-k4.png}
    \caption{}
  \end{subfigure}
  \begin{subfigure}[t]{0.3\textwidth}
    \includegraphics[width=\linewidth]{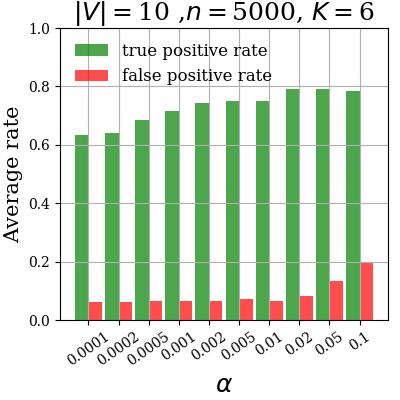}
    \caption{}
  \end{subfigure}
  \\
  \begin{subfigure}[t]{0.3\textwidth}
    \includegraphics[width=\linewidth]{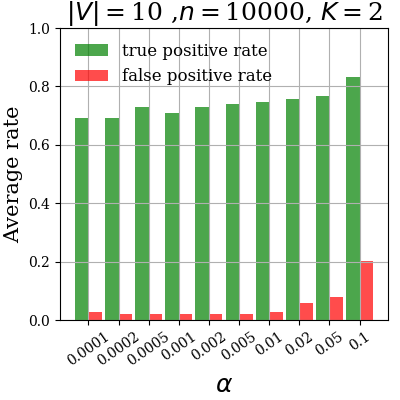}
  \caption{}
  \end{subfigure}
  \begin{subfigure}[t]{0.3\textwidth}
    \includegraphics[width=\linewidth]{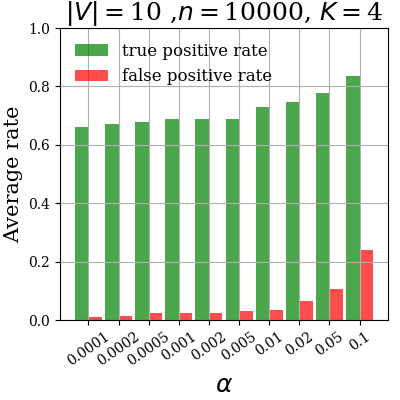}
    \caption{}
  \end{subfigure}
  \begin{subfigure}[t]{0.3\textwidth}
    \includegraphics[width=\linewidth]{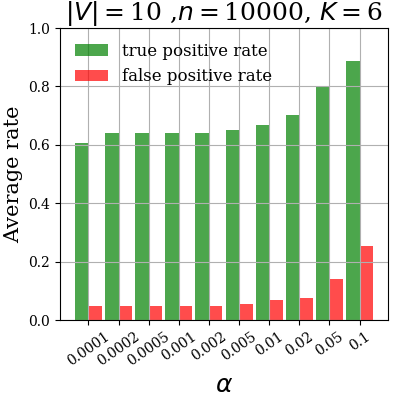}
    \caption{}
  \end{subfigure}
  \caption{True and false positive rates in estimating $V\setminus V_{\textrm{INV}}$ using Proposition~\ref{prop:varying} applied to the PAG $\widehat\cP_{\cup}$ estimated by running FCI on the mxiture data. The figures show the results for $K\in\{2,4,6\}$ and $n\in\{1000,5000,10000\}$. We take $p(j)$ to be uniform.}
  \label{appendix:fig_varying}
\end{figure}

\begin{figure}[H]
  \centering
  \begin{subfigure}[t]{0.3\textwidth}
    \includegraphics[width=\linewidth]{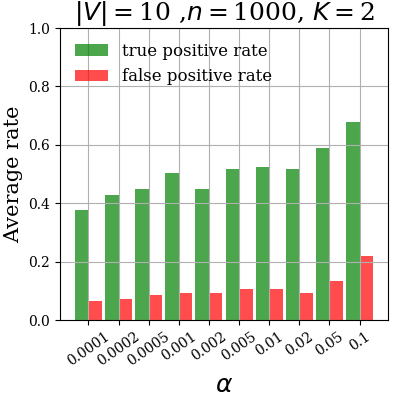}
  \caption{}
  \end{subfigure}
  \begin{subfigure}[t]{0.3\textwidth}
    \includegraphics[width=\linewidth]{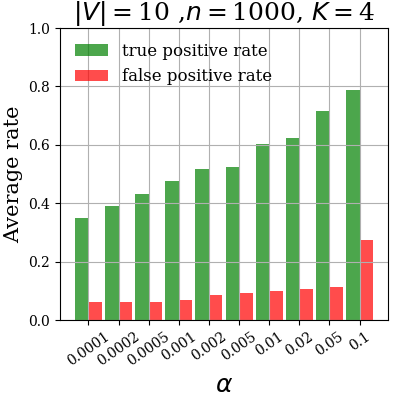}
    \caption{}
  \end{subfigure}
  \begin{subfigure}[t]{0.3\textwidth}
    \includegraphics[width=\linewidth]{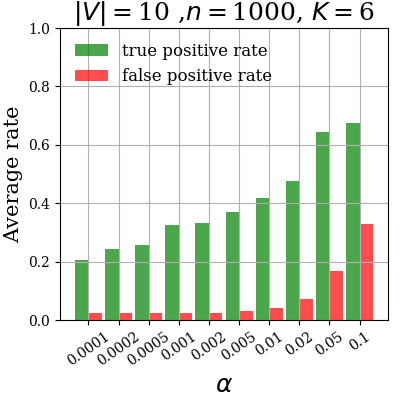}
    \caption{}
  \end{subfigure}
  \\
  \begin{subfigure}[t]{0.3\textwidth}
    \includegraphics[width=\linewidth]{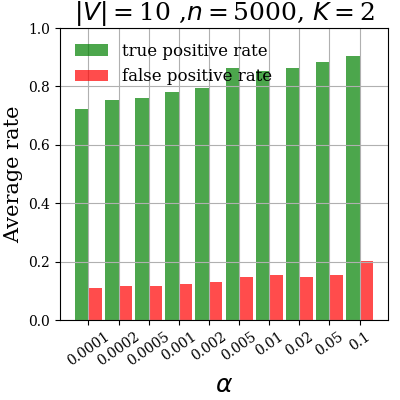}
  \caption{}
  \end{subfigure}
  \begin{subfigure}[t]{0.3\textwidth}
    \includegraphics[width=\linewidth]{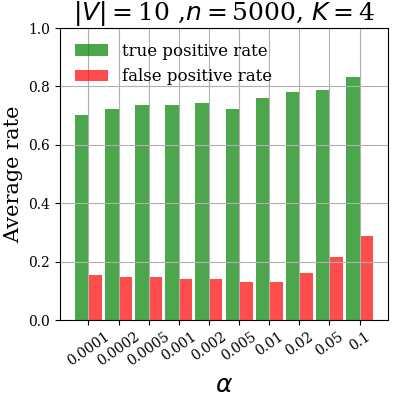}
    \caption{}
  \end{subfigure}
  \begin{subfigure}[t]{0.3\textwidth}
    \includegraphics[width=\linewidth]{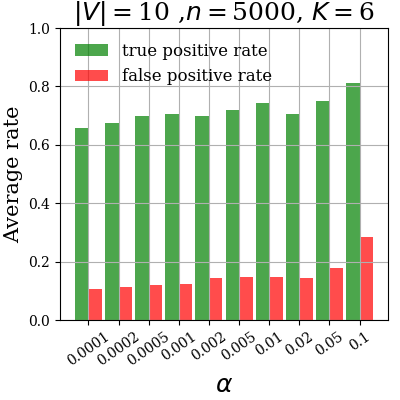}
    \caption{}
  \end{subfigure}
  \\
  \begin{subfigure}[t]{0.3\textwidth}
    \includegraphics[width=\linewidth]{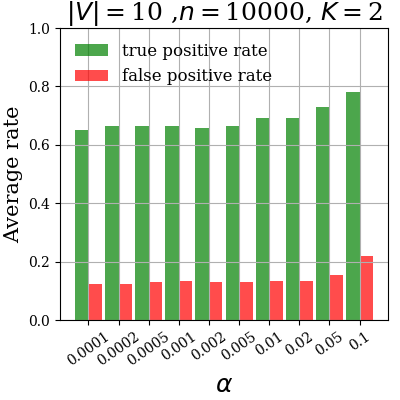}
  \caption{}
  \end{subfigure}
  \begin{subfigure}[t]{0.3\textwidth}
    \includegraphics[width=\linewidth]{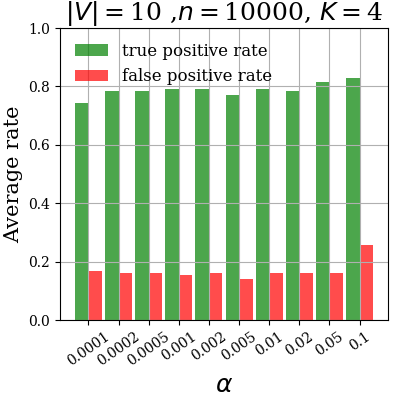}
    \caption{}
  \end{subfigure}
  \begin{subfigure}[t]{0.3\textwidth}
    \includegraphics[width=\linewidth]{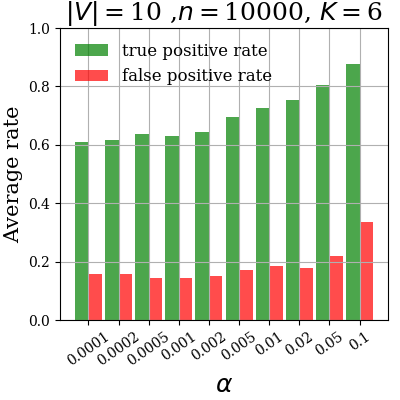}
    \caption{}
  \end{subfigure}
  \caption{True and false positive rates in estimating $V\setminus V_{\textrm{INV}}$ using Proposition~\ref{prop:varying} applied to the PAG $\widehat\cP_{\cup}$ estimated by running FCI on the mxiture data. The figures show the results for $K\in\{2,4,6\}$ and $n\in\{1000,5000,10000\}$. We take $p(j)$ to be Dirichlet with parameter $2$}
  \label{appendix:fig_varying_dir}
\end{figure}

\begin{figure}[H]
  \centering
  \begin{subfigure}[t]{0.3\textwidth}
    \includegraphics[width=\linewidth]{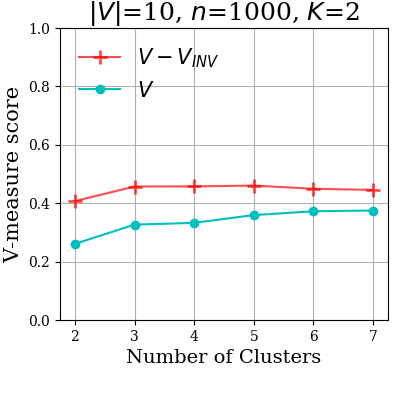}
  \caption{}
  \end{subfigure}
  \begin{subfigure}[t]{0.3\textwidth}
    \includegraphics[width=\linewidth]{figures/CLUSTER-p10-n5000-k2.png}
    \caption{}
  \end{subfigure}
  \begin{subfigure}[t]{0.3\textwidth}
    \includegraphics[width=\linewidth]{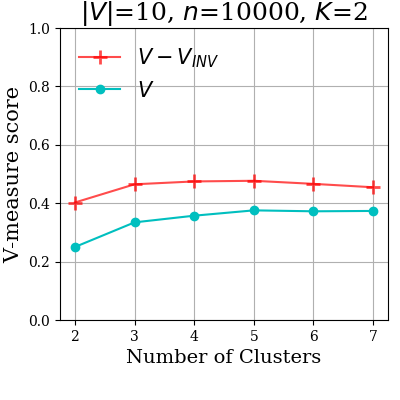}
    \caption{}
  \end{subfigure}
  \\
  \begin{subfigure}[t]{0.3\textwidth}
    \includegraphics[width=\linewidth]{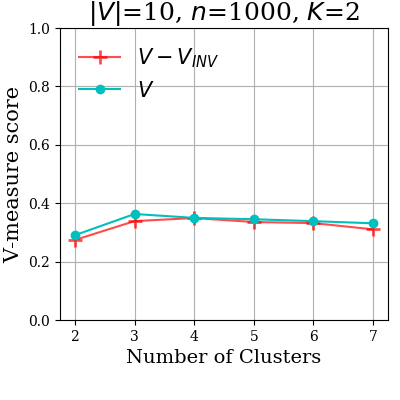}
  \caption{}
  \end{subfigure}
  \begin{subfigure}[t]{0.3\textwidth}
    \includegraphics[width=\linewidth]{figures/nonlocalCLUSTER-p10-n5000-k2.png}
    \caption{}
  \end{subfigure}
  \begin{subfigure}[t]{0.3\textwidth}
    \includegraphics[width=\linewidth]{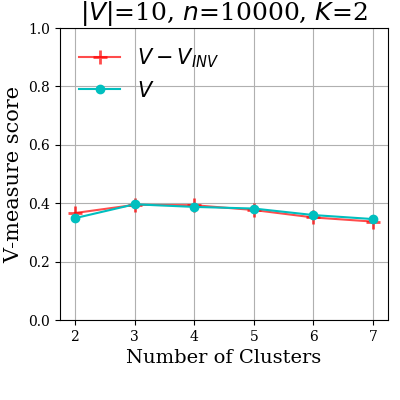}
    \caption{}
  \end{subfigure}
\caption{A comparison of clustering when all the variables are used as features vs.~when only the variables in the estimated set $V\setminus V_{\textrm{INV}}$ are used as features. 
    In generating figures (a), (b) and (c), $V\setminus V_{\textrm{INV}}$ has descendants in the generating model, while in figures (d), (e) and (f), $V\setminus V_{\textrm{INV}}$ has no descendants.}
    \label{appendix:fig_clustering}
\end{figure}

\newpage
\subsection{Real Data}
Here, we present the output of FCI on the T cell mixture data referenced in section~\ref{section:real_data}.
\label{section:tcells_fci}
\begin{figure}[H]
    \centering
    \includegraphics[width=1\linewidth]{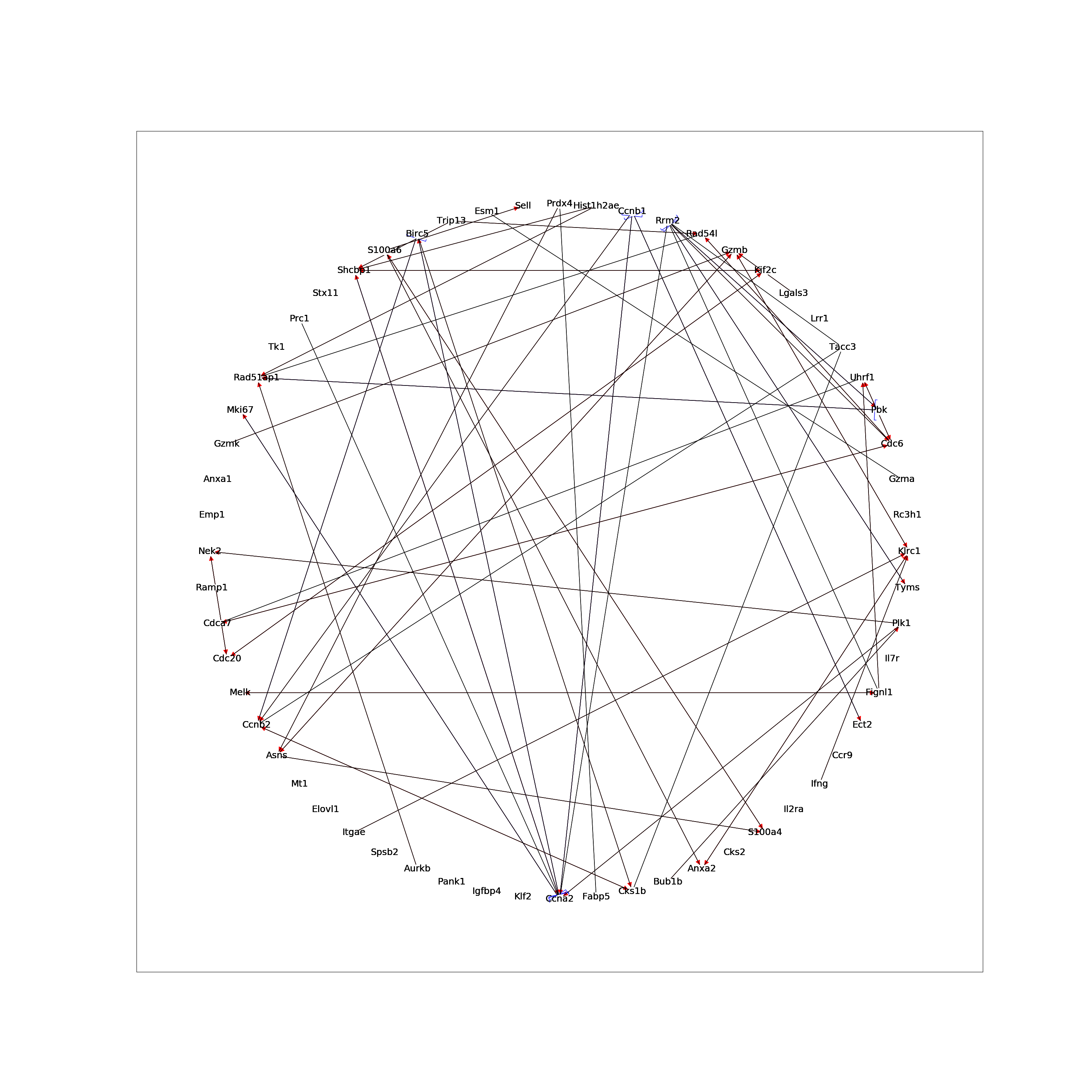}
    \caption{The PAG learned using FCI on the T cell mixture data. The inferred arrowheads are shown in red, while the inferred arrowtails are shown as blue brackets.}
    \label{fig:fci_tcells}
\end{figure}

%
%
%
%

\end{document}